\documentclass[10pt, conference,dvipsnames]{IEEEtran} %

\usepackage{subfig}
\usepackage[utf8]{inputenc} % allow utf-8 input
\usepackage{graphicx,url}
\usepackage{booktabs}       % professional-quality tables
\graphicspath{ {./images/} }
\usepackage{multirow}
\usepackage{hhline}
\usepackage{graphicx}
\usepackage{amssymb,amsmath,mathtools}
\usepackage{amsthm}
\usepackage{bm}

\newtheorem{theorem}{Theorem}
\newtheorem{lemma}{Lemma}
\usepackage{color,soul}
\usepackage[dvipsnames]{xcolor}
\usepackage{bbm}
\usepackage{lipsum,lineno}
\usepackage{float}
\makeatletter
\newif\if@restonecol
\makeatother

\usepackage[linesnumbered,ruled,vlined]{algorithm2e}%[ruled,vlined]{
\usepackage{algpseudocode}
\usepackage{amsmath}
\usepackage{mathrsfs}
\usepackage{fancyhdr} % This should be set AFTER setting up the page geometry
\usepackage{tabularx}
\usepackage{comment} % RS 2020-07-12
\usepackage{tikz}
\usetikzlibrary{shapes.geometric,backgrounds,patterns, trees}
\usetikzlibrary{3d,decorations.text,shapes.arrows,positioning,fit,backgrounds}
\usetikzlibrary{positioning, decorations.pathmorphing, shapes}
\usetikzlibrary{decorations.pathreplacing}
\usetikzlibrary{shapes.geometric,backgrounds,patterns, trees}
\usetikzlibrary{arrows.meta,
                bending,
                intersections,
                quotes,
                shapes.geometric}
              \usetikzlibrary{automata, positioning}
\usetikzlibrary{shapes,arrows}
\usetikzlibrary{arrows.meta}
\usetikzlibrary{positioning}
\tikzset{set/.style={draw,circle,inner sep=0pt,align=center}}
\usetikzlibrary{automata, positioning}
  \usetikzlibrary{shapes,shadows}
  \tikzstyle{abstractbox} = [draw=black, fill=white, rectangle,
  inner sep=10pt, style=rounded corners, drop shadow={fill=black,
  opacity=1}]
\tikzstyle{abstracttitle} =[fill=white]
\usetikzlibrary{calc,positioning,shapes.geometric}
\usetikzlibrary{arrows.meta,arrows}
\DeclareMathOperator*{\argmax}{arg\,max}

\usetikzlibrary{matrix}
\tikzstyle{cblue}=[circle, draw, thin,fill=cyan!20, scale=0.8]
\tikzstyle{qgre}=[rectangle, draw, thin,fill=green!20, scale=0.8]
\tikzstyle{rpath}=[ultra thick, red, opacity=0.4]
\tikzstyle{legend_isps}=[rectangle, rounded corners, thin,
                       fill=gray!20, text=blue, draw]

\tikzstyle{legend_overlay}=[rectangle, rounded corners, thin,
                           top color= white,bottom color=green!25,
                           minimum width=2.5cm, minimum height=0.8cm,
                           pinegreen]
\tikzstyle{legend_phytop}=[rectangle, rounded corners, thin,
                          top color= white,bottom color=cyan!25,
                          minimum width=2.5cm, minimum height=0.8cm,
                          royalblue]
\tikzstyle{legend_general}=[rectangle, rounded corners, thin,
                          top color= white,bottom color=lavander!25,
                          minimum width=2.5cm, minimum height=0.8cm,
                          violet]
\usetikzlibrary{matrix}
\colorlet{myRed}{red!20}
\tikzset{
  rows/.style 2 args={/utils/temp/.style={row ##1/.append style={nodes={#2}}},
    /utils/temp/.list={#1}},
  columns/.style 2 args={/utils/temp/.style={column ##1/.append style={nodes={#2}}},
    /utils/temp/.list={#1}}}
\usetikzlibrary{backgrounds,calc,shadings,shapes.arrows,shapes.symbols,shadows}
\definecolor{switch}{HTML}{006996}

\makeatletter
\pgfkeys{/pgf/.cd,
  parallelepiped offset x/.initial=2mm,
  parallelepiped offset y/.initial=2mm
}
\pgfdeclareshape{parallelepiped}
{
  \inheritsavedanchors[from=rectangle] % this is nearly a rectangle
  \inheritanchorborder[from=rectangle]
  \inheritanchor[from=rectangle]{north}
  \inheritanchor[from=rectangle]{north west}
  \inheritanchor[from=rectangle]{north east}
  \inheritanchor[from=rectangle]{center}
  \inheritanchor[from=rectangle]{west}
  \inheritanchor[from=rectangle]{east}
  \inheritanchor[from=rectangle]{mid}
  \inheritanchor[from=rectangle]{mid west}
  \inheritanchor[from=rectangle]{mid east}
  \inheritanchor[from=rectangle]{base}
  \inheritanchor[from=rectangle]{base west}
  \inheritanchor[from=rectangle]{base east}
  \inheritanchor[from=rectangle]{south}
  \inheritanchor[from=rectangle]{south west}
  \inheritanchor[from=rectangle]{south east}
  \backgroundpath{
    \southwest \pgf@xa=\pgf@x \pgf@ya=\pgf@y
    \northeast \pgf@xb=\pgf@x \pgf@yb=\pgf@y
    \pgfmathsetlength\pgfutil@tempdima{\pgfkeysvalueof{/pgf/parallelepiped
      offset x}}
    \pgfmathsetlength\pgfutil@tempdimb{\pgfkeysvalueof{/pgf/parallelepiped
      offset y}}
    \def\ppd@offset{\pgfpoint{\pgfutil@tempdima}{\pgfutil@tempdimb}}
    \pgfpathmoveto{\pgfqpoint{\pgf@xa}{\pgf@ya}}
    \pgfpathlineto{\pgfqpoint{\pgf@xb}{\pgf@ya}}
    \pgfpathlineto{\pgfqpoint{\pgf@xb}{\pgf@yb}}
    \pgfpathlineto{\pgfqpoint{\pgf@xa}{\pgf@yb}}
    \pgfpathclose
    \pgfpathmoveto{\pgfqpoint{\pgf@xb}{\pgf@ya}}
    \pgfpathlineto{\pgfpointadd{\pgfpoint{\pgf@xb}{\pgf@ya}}{\ppd@offset}}
    \pgfpathlineto{\pgfpointadd{\pgfpoint{\pgf@xb}{\pgf@yb}}{\ppd@offset}}
    \pgfpathlineto{\pgfpointadd{\pgfpoint{\pgf@xa}{\pgf@yb}}{\ppd@offset}}
    \pgfpathlineto{\pgfqpoint{\pgf@xa}{\pgf@yb}}
    \pgfpathmoveto{\pgfqpoint{\pgf@xb}{\pgf@yb}}
    \pgfpathlineto{\pgfpointadd{\pgfpoint{\pgf@xb}{\pgf@yb}}{\ppd@offset}}
  }
}

\makeatletter
\tikzset{anchor/.append code=\let\tikz@auto@anchor\relax,
  add font/.code=%
    \expandafter\def\expandafter\tikz@textfont\expandafter{\tikz@textfont#1},
  left delimiter/.style 2 args={append after command={\tikz@delimiter{south east}
    {south west}{every delimiter,every left delimiter,#2}{south}{north}{#1}{.}{\pgf@y}}}}
\tikzstyle{sms} = [rectangle callout, draw,very thick, rounded corners, minimum height=20pt]
\makeatletter
\tikzset{anchor/.append code=\let\tikz@auto@anchor\relax,
  add font/.code=%
    \expandafter\def\expandafter\tikz@textfont\expandafter{\tikz@textfont#1},
  left delimiter/.style 2 args={append after command={\tikz@delimiter{south east}
    {south west}{every delimiter,every left delimiter,#2}{south}{north}{#1}{.}{\pgf@y}}}}
\tikzstyle{sms} = [rectangle callout, draw,very thick, rounded corners, minimum height=20pt]
\usetikzlibrary{positioning,calc}

\tikzset{l3 switch/.style={
    parallelepiped,fill=switch, draw=white,
    minimum width=0.75cm,
    minimum height=0.75cm,
    parallelepiped offset x=1.75mm,
    parallelepiped offset y=1.25mm,
    path picture={
      \node[fill=white,
        circle,
        minimum size=6pt,
        inner sep=0pt,
        append after command={
          \pgfextra{
            \foreach \angle in {0,45,...,360}
            \draw[-latex,fill=white] (\tikzlastnode.\angle)--++(\angle:2.25mm);
          }
        }
      ]
       at ([xshift=-0.75mm,yshift=-0.5mm]path picture bounding box.center){};
    }
  },
  ports/.style={
    line width=0.3pt,
    top color=gray!20,
    bottom color=gray!80
  },
  rack switch/.style={
    parallelepiped,fill=white, draw,
    minimum width=1.25cm,
    minimum height=0.25cm,
    parallelepiped offset x=2mm,
    parallelepiped offset y=1.25mm,
    xscale=-1,
    path picture={
      \draw[top color=gray!5,bottom color=gray!40]
      (path picture bounding box.south west) rectangle
      (path picture bounding box.north east);
      \coordinate (A-west) at ([xshift=-0.2cm]path picture bounding box.west);
      \coordinate (A-center) at ($(path picture bounding box.center)!0!(path
        picture bounding box.south)$);
      \foreach \x in {0.275,0.525,0.775}{
        \draw[ports]([yshift=-0.05cm]$(A-west)!\x!(A-center)$)
          rectangle +(0.1,0.05);
        \draw[ports]([yshift=-0.125cm]$(A-west)!\x!(A-center)$)
          rectangle +(0.1,0.05);
       }
      \coordinate (A-east) at (path picture bounding box.east);
      \foreach \x in {0.085,0.21,0.335,0.455,0.635,0.755,0.875,1}{
        \draw[ports]([yshift=-0.1125cm]$(A-east)!\x!(A-center)$)
          rectangle +(0.05,0.1);
      }
    }
  },
  server/.style={
    parallelepiped,
    fill=white, draw,
    minimum width=0.35cm,
    minimum height=0.75cm,
    parallelepiped offset x=3mm,
    parallelepiped offset y=2mm,
    xscale=-1,
    path picture={
      \draw[top color=gray!5,bottom color=gray!40]
      (path picture bounding box.south west) rectangle
      (path picture bounding box.north east);
      \coordinate (A-center) at ($(path picture bounding box.center)!0!(path
        picture bounding box.south)$);
      \coordinate (A-west) at ([xshift=-0.575cm]path picture bounding box.west);
      \draw[ports]([yshift=0.1cm]$(A-west)!0!(A-center)$)
        rectangle +(0.2,0.065);
      \draw[ports]([yshift=0.01cm]$(A-west)!0.085!(A-center)$)
        rectangle +(0.15,0.05);
      \fill[black]([yshift=-0.35cm]$(A-west)!-0.1!(A-center)$)
        rectangle +(0.235,0.0175);
      \fill[black]([yshift=-0.385cm]$(A-west)!-0.1!(A-center)$)
        rectangle +(0.235,0.0175);
      \fill[black]([yshift=-0.42cm]$(A-west)!-0.1!(A-center)$)
        rectangle +(0.235,0.0175);
    }
  },
}

\usetikzlibrary{calc, shadings, shadows, shapes.arrows}

\tikzset{%
  interface/.style={draw, rectangle, rounded corners, font=\LARGE\sffamily},
  ethernet/.style={interface, fill=yellow!50},% ethernet interface
  serial/.style={interface, fill=green!70},% serial interface
  speed/.style={sloped, anchor=south, font=\large\sffamily},% line speed at edge
  route/.style={draw, shape=single arrow, single arrow head extend=4mm,
    minimum height=1.7cm, minimum width=3mm, white, fill=switch!20,
    drop shadow={opacity=.8, fill=switch}, font=\tiny}% inroute/outroute arrows
}
\newcommand*{\shift}{1.3cm}% For placing the arrows later

\newcommand{\Crossk}{$\mathbin{\tikz [x=1.2ex,y=1.2ex,line width=.1ex, black] \draw (0,0) -- (1,1) (0,1) -- (1,0);}$}%
\newcommand*{\router}[1]{
\begin{tikzpicture}
  \coordinate (ll) at (-3,0.5);
  \coordinate (lr) at (3,0.5);
  \coordinate (ul) at (-3,2);
  \coordinate (ur) at (3,2);
  \shade [shading angle=90, left color=switch, right color=white] (ll)
    arc (-180:-60:3cm and .75cm) -- +(0,1.5) arc (-60:-180:3cm and .75cm)
    -- cycle;
  \shade [shading angle=270, right color=switch, left color=white!50] (lr)
    arc (0:-60:3cm and .75cm) -- +(0,1.5) arc (-60:0:3cm and .75cm) -- cycle;
  \draw [thick] (ll) arc (-180:0:3cm and .75cm)
    -- (ur) arc (0:-180:3cm and .75cm) -- cycle;
  \draw [thick, shade, upper left=switch, lower left=switch,
    upper right=switch, lower right=white] (ul)
    arc (-180:180:3cm and .75cm);
  \node at (0,0.5){\color{blue!60!black}\Huge #1};% The name of the router
  \begin{scope}[yshift=2cm, yscale=0.28, transform shape]
    \node[route, rotate=45, xshift=\shift] {\strut};
    \node[route, rotate=-45, xshift=-\shift] {\strut};
    \node[route, rotate=-135, xshift=\shift] {\strut};
    \node[route, rotate=135, xshift=-\shift] {\strut};
  \end{scope}
\end{tikzpicture}}

\makeatletter
\pgfdeclareradialshading[tikz@ball]{cloud}{\pgfpoint{-0.275cm}{0.4cm}}{%
  color(0cm)=(tikz@ball!75!white);
  color(0.1cm)=(tikz@ball!85!white);
  color(0.2cm)=(tikz@ball!95!white);
  color(0.7cm)=(tikz@ball!89!black);
  color(1cm)=(tikz@ball!75!black)
}
\tikzoption{cloud color}{\pgfutil@colorlet{tikz@ball}{#1}%
  \def\tikz@shading{cloud}\tikz@addmode{\tikz@mode@shadetrue}}
\makeatother

\tikzset{my cloud/.style={
     cloud, draw, aspect=2,
     cloud color={gray!5!white}
  }
}

\renewcommand\thetable{\arabic{table}}
\usepackage{etoolbox}
\makeatletter
\patchcmd{\@makecaption}
  {\scshape}
  {}
  {}
  {}
  \makeatother
\begin{document}
\bstctlcite{MyBSTcontrol}
\title{
%%Learning Intrusion Prevention Policies\\through Solving the Optimal Stopping Problem
Learning Intrusion Prevention Policies\\through Optimal Stopping
%%\thanks{This research was funded by .}
}
\author{\IEEEauthorblockN{Kim Hammar \IEEEauthorrefmark{2}\IEEEauthorrefmark{3} and Rolf Stadler\IEEEauthorrefmark{2}\IEEEauthorrefmark{3}}

 \IEEEauthorblockA{\IEEEauthorrefmark{2}
Division of Network and Systems Engineering, KTH Royal Institute of Technology, Sweden
 }
 \IEEEauthorblockA{\IEEEauthorrefmark{3} KTH Center for Cyber Defense and Information Security, Sweden \\
  \newline
Email: \{kimham, stadler\}@kth.se
\\
\today
}
}
\date{June 11, 2021}
\IEEEoverridecommandlockouts
\IEEEpubid{\makebox[\columnwidth]{978-3-903176-31-7 \copyright2021 IFIP \hfill} \hspace{\columnsep}\makebox[\columnwidth]{ }}
\maketitle
\thispagestyle{plain}
\pagestyle{plain}
\IEEEpubidadjcol
%To add page numbers
%%\thispagestyle{empty}
%%\pagenumbering{gobble}
%%\pagestyle{plain}
%%\raggedbottom
%%\baselineskip
%% \flushbottom
\begin{abstract}
We study automated intrusion prevention using reinforcement learning. In a novel approach, we formulate the problem of intrusion prevention as an optimal stopping problem. This formulation allows us insight into the structure of the optimal policies, which turn out to be threshold based. Since the computation of the optimal defender policy using dynamic programming is not feasible for practical cases, we approximate the optimal policy through reinforcement learning in a simulation environment. To define the dynamics of the simulation, we emulate the target infrastructure and collect measurements. Our evaluations show that the learned policies are close to optimal and that they indeed can be expressed using thresholds.
\end{abstract}
\section{Introduction}
An organization's security strategy has traditionally been defined, implemented, and updated by domain experts \cite{int_prevention}. Although this approach can provide basic security for an organization's communication and computing infrastructure, a growing concern is that infrastructure update cycles become shorter and attacks increase in sophistication. Consequently, the security requirements become increasingly difficult to meet. As a response, significant efforts are made to automate security processes and functions.
Over the last years, research directions emerged to automatically find and update security policies. One such direction aims at automating the creation of threat models for a given infrastructure \cite{mal_pontus}. A second direction focuses on evolutionary processes that produce novel exploits and corresponding defenses \cite{armsrace_malware}. In a third direction, the interaction between an attacker and a defender is modeled as a game, which allows attack and defense policies to be analyzed and sometimes constructed using game theory \cite{nework_security_alpcan, serkan_gyorgy_game}. In a fourth direction, statistical tests are used to detect attacks \cite{tartakovsky_1}. Further, the evolution of an infrastructure and the actions of a defender is studied using the framework of dynamical systems. This framework allows optimal policies to be obtained using methods from control theory \cite{malware_oc} or dynamic programming \cite{dp_security_1,miehling_attack_graph}. In all of the above directions, machine learning techniques are often applied to estimate model parameters and policies \cite{hammar_stadler, elderman}.

Many activities center around modeling the infrastructure as a discrete-time dynamical system in the form of a Markov Decision Process (MDP). Here, the possible actions of the defender are defined by the action space of the MDP, the defender policy is determined by the actions that the defender takes in different states, and the security objective is encoded in the reward function, which the defender tries to optimize.

To find the optimal policy in an MDP, two main methods are used: dynamic programming and reinforcement learning. The advantage of dynamic programming is that it has a strong theoretical grounding and oftentimes allows to derive properties of the optimal policy \cite{bert05,krishnamurthy_2016}. The disadvantage is that it requires complete knowledge of the MDP, including the transition probabilities. In addition, the computational overhead is high, which makes it infeasible to compute the optimal policy for all but simple configurations \cite{BertsekasTsitsiklis96,krishnamurthy_2016,dp_security_1}. Alternatively, reinforcement learning enables \textit{learning} the dynamics of the model through exploration. With the reinforcement learning approach, it is often possible to compute close approximations of the optimal policy for non-trivial configurations \cite{hammar_stadler, deep_rl_cyber_sec, BertsekasTsitsiklis96, rl_bible}. As a drawback, however, theoretical insights into the structure of the optimal policy generally remain elusive.

In this paper, we study an intrusion prevention use case that involves the IT infrastructure of an organization. The operator of this infrastructure, which we call the defender, takes measures to protect it against a possible attacker while, at the same time, providing a service to a client population. The infrastructure includes a public gateway through which the clients access the service and which also is open to a plausible attacker. The attacker decides when to start an intrusion and then executes a sequence of actions that includes reconnaissance and exploits. Conversely, the defender aims at preventing intrusions and maintaining the service to its clients. It monitors the infrastructure and can block outside access to the gateway, an action that disrupts the service but stops any ongoing intrusion. What makes the task of the defender difficult is the fact that it lacks direct knowledge of the attacker's actions and must infer that an intrusion occurs from monitoring data.

We study the use case within the framework of discrete-time dynamical systems. Specifically, we formulate the problem of finding an optimal defender policy as an \textit{optimal stopping problem}, where stopping refers to blocking access to the gateway. Optimal stopping is frequently used to model problems in the fields of finance and communication systems \cite{optimal_stopping_finance,roy_threshold,tartakovsky_1,kurt_rl}. To the best of our knowledge, finding an intrusion prevention policy through solving an optimal stopping problem is a novel approach.

By formulating intrusion prevention as an optimal stopping problem, we know from the theory of dynamic programming that the optimal policy can be expressed through a threshold that is obtained from observations, i.e. from infrastructure measurements \cite{krishnamurthy_2016,bert05}. This contrasts with prior works that formulate the problem using a general MDP, which does not allow insight into the structure of optimal policies \cite{hammar_stadler,elderman,oslo_uni,schwartz_2020}.

To account for the fact that the defender only has access to a limited number of measurements and cannot directly observe the attacker, we model the optimal stopping problem with a Partially Observed Markov Decision Process (POMDP). We obtain the defender policies by simulating a series of POMDP episodes in which an intrusion takes place and where the defender continuously updates its policy based on outcomes of previous episodes. To update the policy, we use a state-of-the-art reinforcement learning algorithm. This approach enables us to find effective defender policies despite the uncertainty about the attacker's behavior and despite the large state space of the model.

We validate our approach to intrusion prevention for a non-trivial infrastructure configuration and two attacker profiles. Through extensive simulation, we demonstrate that the learned defender policies indeed are threshold based, that they converge quickly, and that they are close to optimal.

We make two contributions with this paper. First, we formulate the problem of intrusion prevention as a problem of optimal stopping. This novel approach allows us a) to derive properties of the optimal defender policy using results from dynamic programming and b) to use reinforcement learning techniques to approximate the optimal policy for a non-trivial configuration. Second, we instantiate the simulation model with measurements collected from an emulation of the target infrastructure, which reduces the assumptions needed to construct the simulation model and narrows the gap between a simulation episode and a scenario playing out in a real system. This addresses a limitation of related work that rely on abstract assumptions to construct the simulation model \cite{hammar_stadler,elderman,oslo_uni,schwartz_2020}.
%%that are not based on system measurements
%%and addresses a limitation of prior research in this area \cite{hammar_stadler,elderman,oslo_uni,schwartz_2020}.
%%which narrows the gap between a simulation episode and a scenario playing out in a real system
%%This also addresses a limitation of the prior work that rely on simulation models that are not based on system measurements
%The remainder of this paper is structured as follows. We first introduce the use case and the theoretical background on Markov Decision Processes (MDPs), reinforcement learning, and optimal stopping. Then, we present our formal model together with our reinforcement learning approach. Subsequently, we derive the threshold property of the optimal policy. We then describe how we emulate the target infrastructure and how we instantiate the simulation. Following that, we evaluate our approach for learning defender policies using simulations. In the evaluation, we validate the threshold property of the learned policies and compare them with the optimal policy. Lastly, we describe how our work relates to prior research, present our conclusions, and provide directions for future work.
\begin{figure}
  \centering
  \scalebox{0.93}{
    \input{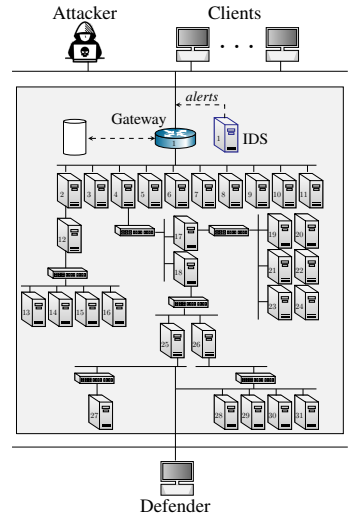}
    }
    \caption{The IT infrastructure and the actors in the use case.}
    \label{fig:system2}
\end{figure}
\section{The Intrusion Prevention Use Case}\label{sec:use_case}
We consider an intrusion prevention use case that involves the IT infrastructure of an organization. The operator of this infrastructure, which we call the defender, takes measures to protect it against an attacker while, at the same time, providing a service to a client population (Fig. \ref{fig:system2}). The infrastructure includes a set of servers that run the service and an intrusion detection system (IDS) that logs events in real-time. Clients access the service through a public gateway, which also is open to the attacker.

We assume that the attacker intrudes into the infrastructure through the gateway, performs reconnaissance, and exploits found vulnerabilities, while the defender continuously monitors the infrastructure through accessing and analyzing IDS statistics and login attempts at the servers. The defender has a single action to stop the attacker, which involves blocking all outside access to the gateway. As a consequence of this action, the service as well as any ongoing intrusion are disrupted.

When deciding whether to block the gateway, the defender must balance two objectives: to maintain the service to its clients and to keep a possible attacker out of the infrastructure. The optimal policy for the defender is to maintain service until the moment when the attacker enters through the gateway, at which time the gateway must be blocked. The challenge for the defender is to identify the precise time when this moment occurs.

In this work, we model the attacker as an agent that starts the intrusion at a random point in time and then takes a predefined sequence of actions, which includes reconnaissance to explore the infrastructure and exploits to compromise the servers.

We study the use case from the defender’s perspective. The evolution of the system state and the actions by the defender are modeled with a discrete-time Partially Observed Markov Decision Process (POMDP). The reward function of this process encodes the benefit of maintaining service and the loss of being intruded. Finding an optimal defender policy thus means maximizing the expected reward. To find an optimal policy, we solve an \textit{optimal stopping problem}, where the stopping action refers to blocking the gateway.
\section{Theoretical Background}\label{sec:theory_background}
This section contains background information on Markov decision processes, reinforcement learning, and optimal stopping.
\subsection{Markov Decision Processes}\label{sec:mdps}
A Markov Decision Process (MDP) models the control of a discrete-time dynamical system and is defined by a seven-tuple $\mathcal{M} = \langle \mathcal{S}, \mathcal{A}, \mathcal{P}^{a_t}_{s_t,s_{t+1}}, \mathcal{R}^{a_t}_{s_t,s_{t+1}}, \gamma, \rho_1, T \rangle$ \cite{bellman1957markovian,puterman}. $\mathcal{S}$ denotes the set of states and $\mathcal{A}$ denotes the set of actions. $\mathcal{P}^{a_t}_{s_t,s_{t+1}}$ refers to the probability of transitioning from state $s_t$ to state $s_{t+1}$ when taking action $a_t$ (Eq. \ref{eq:mdp_prob_1}), which has the Markov property $\mathbb{P}\left[s_{t+1}|s_t\right] = \mathbb{P}\left[s_{t+1}| s_1, \hdots, s_t\right]$. Similarly, $\mathcal{R}^{a_t}_{s_t,s_{t+1}} \in \mathbb{R}$ is the expected reward when taking action $a_t$ and transitioning from state $s_t$ to state $s_{t+1}$ (Eq. \ref{eq:mdp_reward_fun_1}). If $\mathcal{P}^{a_t}_{s_t,s_{t+1}}$ and $\mathcal{R}^{a_t}_{s_t,s_{t+1}}$ are independent of the time-step $t$, the MDP is said to be \textit{stationary}. Finally, $\gamma \in \left(0,1\right]$ is the discount factor, $\rho_1 : \mathcal{S} \rightarrow [0,1]$ is the initial state distribution, and $T$ is the time horizon.
\begin{align}
&\mathcal{P}^{a_t}_{s_t,s_{t+1}} = \mathbb{P}\left[s_{t+1}| s_t, a_t\right]\label{eq:mdp_prob_1}\\
& \mathcal{R}^{a_t}_{s_t,s_{t+1}} = \mathbb{E}\left[r_{t+1}| a_t,  s_t, s_{t+1}\right] \label{eq:mdp_reward_fun_1}
\end{align}
The system evolves in discrete time-steps from $t=1$ to $t=T$, which constitute one \textit{episode} of the system.

A Partially Observed Markov Decision Process (POMDP) is an extension of an MDP \cite{howard_mdps,krishnamurthy_2016}. In contrast to an MDP, in a POMDP the states are not directly observable. A POMDP is defined by a nine-tuple $\mathcal{M}_{\mathcal{P}} = \langle \mathcal{S}, \mathcal{A}, \mathcal{P}^{a_t}_{s_t,s_{t+1}}, \mathcal{R}^{a_t}_{s_t,s_{t+1}},\gamma, \rho_1, T, \mathcal{O}, \mathcal{Z}\rangle$. The first seven elements define an MDP. $\mathcal{O}$ denotes the set of observations and $\mathcal{Z}(o_{t+1},s_{t+1},a_t) = \mathbb{P}[o_{t+1}|s_{t+1},a_t]$ is the observation function, where $o_{t+1} \in \mathcal{O}$, $s_{t+1} \in \mathcal{S}$, and $a_t \in \mathcal{A}$.

The belief state $b_t \in \mathcal{B}$ is defined as $b_t(s)=\mathbb{P}[s_t=s|h_t]$ for all $s \in \mathcal{S}$. The belief space $\mathcal{B}=\Delta(\mathcal{S})$ is the unit $(|\mathcal{S}|-1)$-simplex \cite{pomdp_belief_optimal,ASTROM1965174}, where $\Delta(\mathcal{S})$ denotes the set of probability distributions over $\mathcal{S}$. $b_t$ is a sufficient statistic of the state $s_t$ based on the history $h_t$ of the initial state distribution, the actions, and the observations: $h_t=(\rho_1,a_1,o_1,\hdots,a_{t-1},o_t) \in \mathcal{H}$. By defining the state at time $t$ to be the belief state $b_t$, a POMDP can be formulated as a continuous-state MDP: $\mathcal{M} = \langle \mathcal{B}, \mathcal{A}, \mathcal{P}^{a_t}_{b_t,b_{t+1}}, \mathcal{R}_{b_t,b_{t+1}}^{a_t}, \gamma, \rho_1, T \rangle$.

The belief state can be computed recursively as follows \cite{krishnamurthy_2016}:
\begin{align}
b_{t+1}(s_{t+1}) &= C\mathcal{Z}(o_{t+1}, s_{t+1}, a_t)\sum_{s_t \in \mathcal{S}}\mathcal{P}^{a_t}_{s_ts_{t+1}}b_t(s_t)\label{eq:belief_upd}
\end{align}
where $C=1/\sum_{s_{t+1} \in S}\mathcal{Z}(o_{t+1},s_{t+1},a_t) \sum_{s_t \in S}\mathcal{P}_{s_t,s_{t+1}}^{a_t}b_t(s_t)$ is a normalizing factor independent of $s_{t+1}$ to make $b_{t+1}$ sum to $1$.
\subsection{The Reinforcement Learning Problem}\label{sec:rl_prob2}
Reinforcement learning deals with the problem of choosing a sequence of actions for a sequentially observed state variable to maximize a reward function \cite{BertsekasTsitsiklis96,rl_bible}. This problem can be modeled with an MDP if the state space is observable, or with a POMDP if the state space is not fully observable.

In the context of an MDP, a policy is defined as a function $\pi: \{1,\hdots, T\} \times \mathcal{S} \rightarrow \Delta(\mathcal{A})$, where $\Delta(\mathcal{A})$ denotes the set of probability distributions over $\mathcal{A}$. In the case of a POMDP, a policy is defined as a function $\pi: \mathcal{H} \rightarrow \Delta(\mathcal{A})$, or, alternatively, as a function $\pi: \{1,\hdots, T\} \times \mathcal{B} \rightarrow \Delta(\mathcal{A})$. In both cases, a policy is called \textit{stationary} if it is independent of the time-step $t$.

An optimal policy $\pi^{*}$ is a policy that maximizes the expected discounted cumulative reward over the time horizon $T$:
\begin{align}
\pi^{*} &\in \argmax_{\pi \in \Pi} \mathbb{E}_{\pi}\left[\sum_{t=1}^{T}\gamma^{t-1}r_{t}\right] \label{eq:rl_prob}
\end{align}
where $\Pi$ is the policy space, $\gamma$ is the discount factor, $r_t$ is the reward at time $t$, and $\mathbb{E}_{\pi}$ denotes the expectation under $\pi$.

It is well known that optimal \textit{deterministic} policies exist for MDPs and POMDPs with finite state and action spaces \cite{puterman,krishnamurthy_2016}. Further, for stationary MDPs and POMDPs with infinite or random time-horizons, optimal \textit{stationary} policies exist \cite{puterman,krishnamurthy_2016}.

The Bellman equations relate any optimal policy $\pi^{*}$ to the two value functions $V^{*} : \mathcal{S} \rightarrow \mathbb{R}$ and $Q^{*}: \mathcal{S} \times \mathcal{A} \rightarrow \mathbb{R}$, where $\mathcal{S}$ and $\mathcal{A}$ are state and action spaces of an MDP \cite{bellman_eq}:
\begin{align}
V^{*}(s_t) &= \displaystyle\max_{a_t\in \mathcal{A}} \mathbb{E}\big[r_{t+1} + \gamma V^{*}(s_{t+1}) | s_t, a_t\big]\label{eq:bellman_eq_31} \\
Q^{*}(s_t,a_t) &= \mathbb{E}\big[r_{t+1} + \gamma V^{*}(s_{t+1}) | s_t, a_t\big] \label{eq:bellman_eq_33}\\
\pi^{*}(s_t) &\in \argmax_{a_t\in \mathcal{A}} Q^*(s_t,a_t)\label{eq:bellman_eq_34}
\end{align}
Here, $V^{*}(s_t)$ and $Q^{*}(s_t,a_t)$ denote the expected cumulative discounted reward under $\pi^{*}$ for each state and state-action pair, respectively. In the case of a POMDP, the Bellman equations contain $b_t$ instead of $s_t$. Solving the Bellman equations (Eqs. \ref{eq:bellman_eq_31}-\ref{eq:bellman_eq_33}) means computing the value functions from which an optimal policy can be obtained (Eq. \ref{eq:bellman_eq_34}).

Two principal methods are used for finding an optimal policy in a MDP or POMDP: dynamic programming and reinforcement learning.

First, the dynamic programming method (e.g. value iteration \cite{bert05,puterman}) assumes complete knowledge of the seven-tuple MDP or the nine-tuple POMDP and obtains an optimal policy by solving the Bellman equations iteratively (Eq. \ref{eq:bellman_eq_34}), with polynomial time-complexity per iteration for MDPs and PSPACE-complete time-complexity for POMDPs \cite{pspace_complexity}.

Second, the reinforcement learning method computes or approximates an optimal policy without requiring complete knowledge of the transition probabilities or observation probabilities of the MDP or POMDP. Three classes of reinforcement learning algorithms exist: \textit{value-based algorithms}, which approximate solutions to the Bellman equations (e.g. Q-learning \cite{watkins_thesis}); \textit{policy-based algorithms}, which directly search through policy space using gradient-based methods (e.g. Proximal Policy Optimization (PPO) \cite{ppo}); and \textit{model-based algorithms}, which learn the transition or observation probabilities of the MDP or POMDP (e.g. Dyna-Q \cite{rl_bible}). The three algorithm types can also be combined, e.g. through \textit{actor-critic} algorithms, which are mixtures of value-based and policy-based algorithms \cite{rl_bible}. In contrast to dynamic programming algorithms, reinforcement learning algorithms generally have no guarantees to converge to an optimal policy except for the tabular case \cite{jaakola_convergence_Q_NIPS1993, robbins_monro}.
\subsection{Markovian Optimal Stopping Problems}\label{sec:optimal_stopping}
Optimal stopping is a classical problem in statistics with a developed theory \cite{wald,shirayev,stopping_book_1,bert05,puterman}. Example applications of this problem are: selling an asset \cite{bert05}, detecting distribution changes \cite{tartakovsky_1}, machine replacement \cite{krishnamurthy_2016}, valuing a financial option \cite{optimal_stopping_finance}, and choosing a candidate for a job (the secretary problem) \cite{puterman}.

Different versions of the problem can be found in the literature. Including, discrete-time and continuous time, finite horizon and infinite horizon, single-stop and multiple stops, fully observed and partially observed, independent and dependent, and Markovian and non-Markovian. Consequently, there are also different solution methods, most prominent being the martingale approach \cite{stopping_book_1} and the Markovian approach \cite{shirayev,bert05,puterman}. In this paper, we consider a partially observed Markovian optimal stopping problem in discrete-time with a finite horizon and a single stop action.

A Markovian optimal stopping problem can be seen as a specific kind of MDP or POMDP where the state of the environment evolves as a discrete-time Markov process $(s_t)_{t=1}^{T}$ which is either fully or partially observed \cite{puterman,krishnamurthy_2016}. At each time-step $t$ of this decision process, two actions are available: ``stop'' and ``continue''. The \textit{stop} action causes the interaction with the environment to stop and yields a stopping-reward. Conversely, the \textit{continue} action causes the environment to evolve to the next time-step and yields a continuation-reward. The \textit{stopping time} $\tau$ is a random variable dependent on $s_1,\hdots,s_t$ and independent of $s_{t+1},\hdots s_{T}$ \cite{stopping_book_1}.

The objective is to find a stopping policy $\pi(s_t) \mapsto \{S,C\}$ that maximizes the expected reward, where $\pi(s_t) = S$ indicates a stopping action. This induces the following maximization at each time-step before stopping (the Bellman equation \cite{bellman_eq}):
\begin{align}
\max \Bigg[\underbrace{\mathbb{E}\left[\mathcal{R}^{S}_{ss^{\prime}}\right]}_{\text{stopping reward}}, \underbrace{\mathbb{E}\left[\mathcal{R}^{C}_{ss^{\prime}} + \gamma V^{*}(s^{\prime})\right]}_{\text{continuation reward}}\Bigg]\label{eq:optimal_stopping_1}
\end{align}
To solve the maximization above, standard solution methods for MDPs and POMDPs can be applied, such as dynamic programming and reinforcement learning \cite{bert05,roy_threshold}. Further, the solution can be characterized using dynamic programming theory as the least excessive (or superharmonic) majorant of the reward function, or using martingale theory as the Snell envelope of the reward function \cite{Snell1952TAMS,stopping_book_1}.
%%As there are only two actions, the solution is completely characterized by the \textit{stopping set} $\{s \in \mathcal{S} : \mathbb{E}[\mathcal{R}^{S}_{ss^{\prime}}] > \mathbb{E}[\mathcal{R}^{C}_{ss^{\prime}} + \gamma V^{*}(s^{\prime})]\}$, i.e., the set of states where the policy stops. Under specific conditions, the optimal policy can have a specific structure, such as a threshold structure, or a \textit{one-step lookahead rule} structure \cite{bert05}.
\section{Formalizing The Intrusion Prevention Use Case and Our Reinforcement Learning Approach}\label{sec:formal_model_2}
In this section, we first formalize the intrusion prevention use case described in Section \ref{sec:use_case} and then we introduce our solution method. Specifically, we first define a POMDP model of the intrusion prevention use case. Then, we describe our reinforcement learning approach to approximate the optimal defender policy. Lastly, we use the theory of dynamic programming to derive the threshold property of the optimal policy.
\subsection{A POMDP Model of the Intrusion Prevention Use Case}
We model the intrusion prevention use case as a partially observed optimal stopping problem where an intrusion starts at a geometrically distributed time and the stopping action refers to blocking the gateway (Fig. \ref{fig:stopping_times}). This type of optimal stopping problem is often referred to as a \textit{quickest change detection} problem \cite{stopping_book_1,shirayev,tartakovsky_1}.

To formalize this model, we use a POMDP. This model includes the state space and the observation space of the defender. It further includes the initial state distribution, the defender actions, the transition probabilities, the observation function, the reward function, and the optimization objective.
\begin{figure}
  \centering
  \scalebox{1.15}{
    \begin{tikzpicture}[fill=white, >=stealth,
    node distance=3cm,
    database/.style={
      cylinder,
      cylinder uses custom fill,
      %%cylinder body fill=yellow!50,
      %%cylinder end fill=yellow!50,
      shape border rotate=90,
      aspect=0.25,
      draw}]

    \tikzset{
node distance = 9em and 4em,
sloped,
   box/.style = {%
    shape=rectangle,
    rounded corners,
    draw=blue!40,
    fill=blue!15,
    align=center,
    font=\fontsize{12}{12}\selectfont},
 arrow/.style = {%
    %%draw=blue!30,
    line width=0.1mm,% <-- select desired width
    -{Triangle[length=5mm,width=2mm]},
    shorten >=1mm, shorten <=1mm,
    font=\fontsize{8}{8}\selectfont},
}

\draw[->, color=black] (0, 0) to (5.9, 0);
\draw[-, color=black] (0, -0.12) to (0, 0.12);
\draw[-, color=black] (0.5, -0.12) to (0.5, 0.12);
\draw[-, color=black] (1, -0.12) to (1, 0.12);
\draw[-, color=black] (1.5, -0.12) to (1.5, 0.12);
\draw[-, color=black] (2, -0.12) to (2, 0.12);
\draw[-, color=black] (2.5, -0.12) to (2.5, 0.12);
\draw[-, color=black] (3, -0.12) to (3, 0.12);
\draw[-, color=black] (3.5, -0.12) to (3.5, 0.12);
\draw[-, color=black] (4, -0.12) to (4, 0.12);
\draw[-, color=black] (4.5, -0.12) to (4.5, 0.12);
\draw[-, color=black] (5, -0.12) to (5, 0.12);
\draw[-, color=black] (5.5, -0.12) to (5.5, 0.12);

\draw[-, color=black, dashed] (3.5, 1) to (3.5, -1);

\node[inner sep=0pt,align=center, scale=1] (time) at (3.5,0)
{
\Crossk
};

\node[inner sep=0pt,align=center, scale=0.75] (time) at (1.9,0.75)
{
Intrusion event
};

\node[inner sep=0pt,align=center, scale=0.75] (time) at (0.1,0.55)
{
time-step $t=1$
};

\draw[->, color=black] (0, 0.43) to (0, 0.16);

\node[inner sep=0pt,align=center, scale=0.75] (time) at (4.65,0.65)
{
Intrusion ongoing
};

\node[inner sep=0pt,align=center, scale=0.75] (time) at (6.1,0)
{
$t$
};

\node[inner sep=0pt,align=center, scale=0.75] (time) at (5.56,-0.35)
{
$t=T$
};

\node[inner sep=0pt,align=center, scale=0.75] (time) at (1,-1)
{
Early stopping times
};
\draw[->, color=black] (1, -0.85) to (0.0, -0.25);
\draw[->, color=black] (1, -0.85) to (0.5, -0.25);
\draw[->, color=black] (1, -0.85) to (1, -0.25);
\draw[->, color=black] (1, -0.85) to (1.5, -0.25);
\draw[->, color=black] (1, -0.85) to (2, -0.25);
\draw[->, color=black] (1, -0.85) to (2.5, -0.25);
\draw[->, color=black] (1, -0.85) to (3, -0.25);

\node[inner sep=0pt,align=left, scale=0.75] (time) at (5,-1.05)
{
  Stopping times that \\
  interrupt the intrusion
};

\draw[->, color=black] (4.55, -0.75) to (3.5, -0.15);
\draw[->, color=black] (4.55, -0.75) to (4, -0.25);
\draw[->, color=black] (4.55, -0.75) to (4.5, -0.25);
\draw[->, color=black] (4.55, -0.75) to (5, -0.25);

\draw [decorate,decoration={brace,amplitude=5pt,mirror,raise=4pt},yshift=0pt,rotate=180, line width=0.25mm]
(-5.45,-0.1) -- (-3.55,-0.1) node [black,midway,xshift=0.1cm] {};

\draw[->, color=black] (2.1,0.55) to (3.3,0.15);

\draw [decorate,decoration={brace,amplitude=5pt,mirror,raise=4pt},yshift=0pt,rotate=180, line width=0.25mm]
(-5.45,-0.75) -- (0,-0.75) node [black,midway,xshift=0.1cm] {};

\node[inner sep=0pt,align=left, scale=0.75] (time) at (2.7,1.25)
{
Episode
};

\end{tikzpicture}
    }
    \caption{Optimal stopping formulation of intrusion prevention; the horizontal axis represents time; $T$ is the time horizon; the episode length is $T-1$; the dashed line shows the intrusion start time; the optimal policy is to stop at the time of intrusion.}
    \label{fig:stopping_times}
  \end{figure}
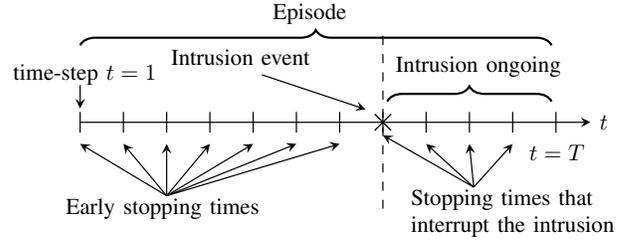
\subsubsection{States $\mathcal{S}$, Initial State Distribution $\rho_1$, and Observations $\mathcal{O}$}
The system state $s_t$ is defined by the intrusion state $i_t \in \{0,1\}$, where $i_t=1$ if an intrusion is ongoing. Further, we introduce a terminal state $\emptyset$, which is reached either when the defender stops or when the attacker completes an intrusion. Thus, $\mathcal{S} = \{0,1\} \cup \emptyset$.

At time $t=1$ no intrusion is in progress. Hence, the initial state distribution is the degenerate distribution $\rho_1(s_1 = 0) = 1$.

The defender has a partial view of the system state and does not know whether an intrusion is in progress. Specifically, if the defender has not stopped, it observes three counters $o_t = (\Delta x_t, \Delta y_t, \Delta z_t)$. The counters are upper bounded, where $\Delta x_t\in \{0, ,\hdots,X_{max}\}$, $\Delta y_t\in \{0, \hdots, Y_{max}\}$, $\Delta z_t\in \{0, \hdots, Z_{max}\}$ denote the number of severe IDS alerts, warning IDS alerts, and login attempts generated during time-step $t$, respectively. Otherwise, if the defender has stopped, it observes $o_t=s_t=\emptyset$. Consequently, $\mathcal{O} = \{0, \hdots, X_{max}\} \times \{0, \hdots, Y_{max}\} \times \{0, \hdots, Z_{max}\} \cup \emptyset$.
\subsubsection{Actions $\mathcal{A}$}\label{sec:actions}
The defender has two actions: ``stop'' ($S$) and ``continue'' ($C$). The action space is thus $\mathcal{A} = \{S,C\}$.
\subsubsection{Transition Probabilities $ \mathcal{P}^a_{ss^{\prime}}$}
We model the start of an intrusion by a Bernoulli process $(Q_t)_{t=1}^{T}$, where $Q_t \sim Ber(p=0.2)$ is a Bernoulli random variable. The time $t$ of the first occurrence of $Q_t=1$ is the change point representing the start time of the intrusion $I_t$, which thus is geometrically distributed, i.e. $I_t \sim Ge(p=0.2)$ (Fig. \ref{fig:reward_fun}). As the geometric distribution has the memoryless property, the intrusion start time is Markovian.
%An intrusion lasts for a finite number of time-steps, denoted by a constant $i_T$. This implies that an intrusion has ended if $t \geq I_t + i_T$.

We define the transition probabilities $\mathcal{P}^{a_t}_{s_ts_{t+1}} =\mathbb{P}\left[s_{t+1}| s_t, a_t\right]$ as follows:
\begin{align}
&\mathbb{P}\left[\emptyset \middle| \cdot ,S \right] = \mathbb{P}\left[\emptyset \middle| \emptyset,\cdot\right]=1 \label{eq:tp_1}\\
%&\mathbb{P}\left[\emptyset \middle| 1 ,\cdot \right]=1 \quad \quad \quad\quad\quad\quad t \geq I_t + i_T  \label{eq:tp_5}\\
&\mathbb{P}\left[0 \middle|0, C \right] = 1-p \label{eq:tp_2}\\
&\mathbb{P}\left[1 \middle|0,C\right]= p \label{eq:tp_3}\\
&\mathbb{P}\left[1 \middle|1, C \right] = 1 \label{eq:tp_4}
\end{align}
All other transitions have probability $0$.

Eq. \ref{eq:tp_1} defines the transition probabilities to the terminal state $\emptyset$. The terminal state is reached when taking the stop action $S$. Eq. \ref{eq:tp_2}-\ref{eq:tp_4} define the transition probabilities when taking the continue action $C$. Eq. \ref{eq:tp_2} captures the case where no intrusion occurs and where $i_{t+1}=i_t=0$; Eq. \ref{eq:tp_3} captures the start of an intrusion where $i_t=0, i_{t+1}=1$; and Eq. \ref{eq:tp_4} describes the case where an intrusion is in progress and $i_{t+1}=i_t = 1$.

\subsubsection{Observation Function $\mathcal{Z}(o^{\prime},s^{\prime},a)$}
We assume that the number of IDS alerts and login attempts generated during a single time-step are random variables $X \sim f_X$, $Y \sim f_Y$, $Z \sim f_Z$, dependent on the intrusion state and defined on the sample spaces $\Omega_X = \{0,1,\hdots, X_{max}\}$, $\Omega_Y = \{0,1,\hdots, Y_{max}\}$, and $\Omega_Z = \{0,1,\hdots, Z_{max}\}$. Consequently, the probability that $\Delta x$ severe alerts, $\Delta y$ warning alerts, and $\Delta z$ login attempts are generated during time-step $t$ is $f_{XYZ}(\Delta x,\Delta y,\Delta z|i_t)$.

We define the observation function $\mathcal{Z}(o^{\prime},s^{\prime},a) =\mathbb{P}[o^{\prime}|s^{\prime},a]$ as follows:
\begin{align}
&\mathcal{Z}\big((\Delta x,\Delta y, \Delta z),i_t,C\big) = f_{XYZ}(\Delta x,\Delta y, \Delta z | i_t)\\
&\mathcal{Z}\big(\emptyset,\emptyset,\cdot\big) = 1
\end{align}
\begin{figure}
  \centering
    \scalebox{0.43}{
      \includegraphics{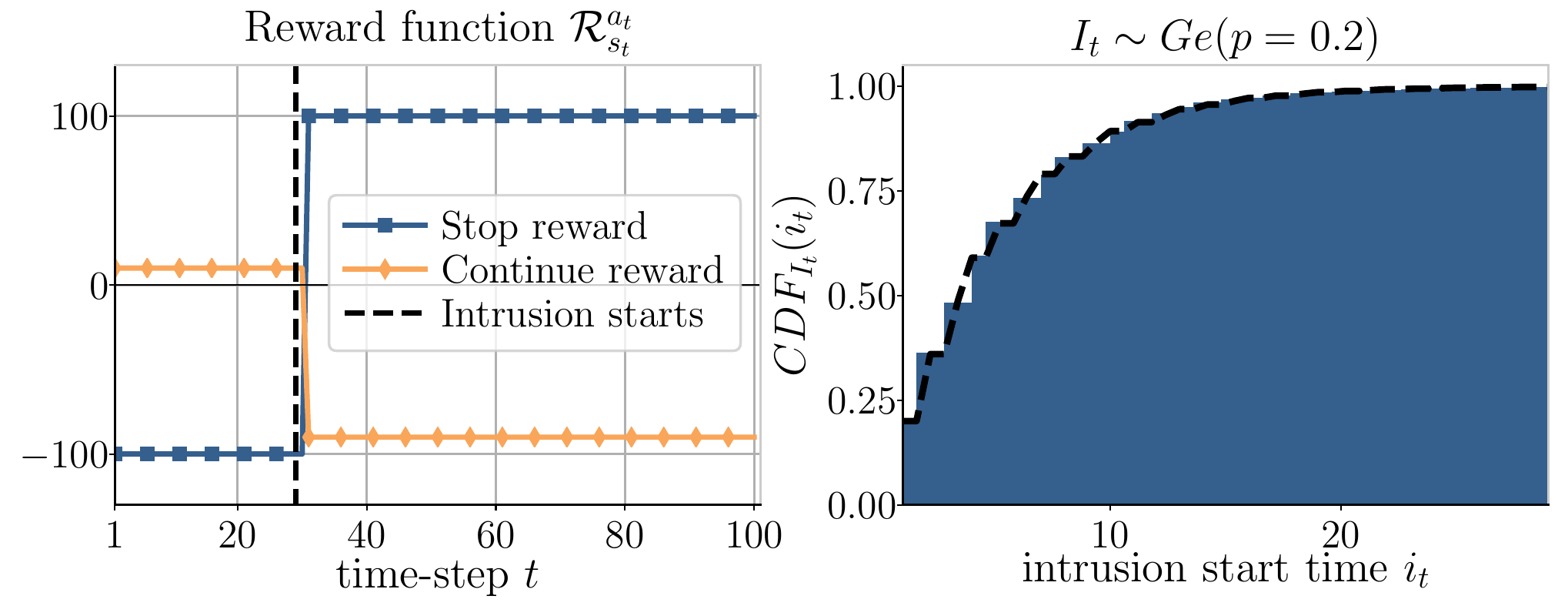}
    }
    \caption{Left: the reward function for the stop and continue actions; the intrusion starts at $t=29$; right: the cumulative distribution function (CDF) of the intrusion start time.}
    \label{fig:reward_fun}
  \end{figure}
\subsubsection{Reward Function $\mathcal{R}_{s}^a$}\label{sec:reward_fun}
The reward function is parameterized by the reward that the defender receives for stopping an intrusion ($R_{st}=100$), the loss of stopping before an intrusion has started ($R_{es}=-100$), the reward for maintaining service ($R_{sla}=10$), and the loss of being intruded ($R_{int}=-100$), respectively.

We define the deterministic reward function $\mathcal{R}^{a_t}_{s_t}= r(s_t, a_t)$ to be (Fig. \ref{fig:reward_fun}):
\begin{align}
&r\left(\emptyset, \cdot\right) = 0 \label{eq:reward_0}\\
&r\left(i_t, S \right) = \mathbbm{1}_{i_t=0} R_{es} + \mathbbm{1}_{i_t=1}R_{st} \label{eq:reward_2}\\
&r\left(i_t, C \right) = R_{sla} + \mathbbm{1}_{i_t=1}R_{int}\label{eq:reward_3}
\end{align}
Eq. \ref{eq:reward_0} states that the reward in the terminal state is zero. Eq. \ref{eq:reward_2} indicates that stopping an intrusion incurs a reward but stopping before an intrusion starts yields a loss, where $S$ is the stop action and $\mathbbm{1}$ is the indicator function. Lastly, as can be seen from Eq. \ref{eq:reward_3}, the defender receives a positive reward for maintaining service and a loss for taking the continue action $C$ while under intrusion. This means that the maximal reward is received if the defender stops when an intrusion starts.
\subsubsection{Time Horizon $T_{\emptyset}$}
The time horizon is defined by the time-step when the terminal state $\emptyset$ is reached, which is a random variable $T_{\emptyset}$. Since we know that the expectation of the intrusion time $I_t$ is finite, we conclude that the horizon is finite for any policy $\pi_{\theta}$ that is guaranteed to stop: $\mathbb{E}_{\pi_{\theta}}\left[T_{\emptyset}\right] < \infty$.
\subsubsection{Policy Space $\Pi_{\theta}$, and Objective $J$}
Since the POMDP is stationary and the time horizon $T_{\emptyset}$ is not pre-determined, it is sufficient to consider stationary policies. Further, although an optimal deterministic policy exists \cite{puterman,krishnamurthy_2016}, we consider stochastic policies to allow smooth optimization. Specifically, we consider the space of stationary stochastic policies $\Pi_{\theta}$ where $\pi_{\theta} \in \Pi_{\theta}$ is a policy $\pi: \mathcal{H} \rightarrow \Delta(\mathcal{A})$, which is parameterized by a vector $\theta \in \mathbb{R}^d$.

%A policy $\pi_{\theta}$ is a function that maps either histories or belief states to actions: $\pi_{\theta} : \mathcal{H} \mapsto \mathcal{A}$ or $\pi_{\theta} : \mathcal{B} \mapsto \mathcal{A}$. Further, a policy $\pi_{\theta}$ is parameterized by a vector $\theta \in \mathbb{R}^d$, where $\pi_{\theta} \in \Pi_{\theta}$.

The optimal policy $\pi^{*}_{\theta} \in \Pi_{\theta}$ maximizes the expected cumulative reward over the random horizon $T_{\emptyset}$:
\begin{align}
J(\theta) = \mathbb{E}_{\pi_{\theta}}\left[\sum_{t=1}^{T_{\emptyset}}\gamma^{t-1}r(s_t, a_t) \right], \text{ } \pi^{*} = \argmax_{\pi_{\theta} \in \Pi_{\theta}}J(\theta)\label{eq:rl_prob2}
\end{align}
where we set the discount factor $\gamma=1$.

Eq. \ref{eq:rl_prob2} defines the objective of the optimal stopping problem. In the following section, we describe our approach for solving this problem using reinforcement learning.
\subsection{Our Reinforcement Learning Approach}\label{sec:rl_approach}
Since the POMDP model is \textit{unknown} to the defender, we use a model-free reinforcement learning approach to approximate the optimal policy. Specifically, we use the state-of-the-art reinforcement learning algorithm PPO \cite{ppo} to learn a policy $\pi_{\theta}: \mathcal{H} \mapsto \mathcal{A}$ that maximizes the objective in Eq. \ref{eq:rl_prob2}.

Due to computational limitations (i.e. finite memory), we summarize the history $h_t=(a_1,o_1,a_2,o_2,\hdots, a_{t-1},o_t)$ by a vector $\hat{h}_t = (x_t, y_t, z_t, t)$, where $x_t, y_t, z_t$ are the accumulated counters of the observations $o_i=(\Delta x_i, \Delta y_i, \Delta z_i)$ for $i=1,\hdots,t$: $x_t=\sum_{i=1}^t\Delta x_i$, $y_t=\sum_{i=1}^t\Delta y_i$, $z_t=\sum_{i=1}^t\Delta z_i$.

PPO implements the \textit{policy gradient} method and uses stochastic gradient ascent with the following gradient \cite{ppo}:
\begin{align}
\nabla_{\theta} J(\theta) &= \mathbb{E}_{\pi_{\theta}}\Bigg[\underbrace{\nabla_{\theta}\log\pi_{\theta}(a|\hat{h})}_{\text{actor}} \underbrace{A^{\pi_{\theta}}(\hat{h},a)}_{\text{critic}}\Bigg] \label{eq:pg_objective}
\end{align}
where $A^{\pi_{\theta}}(\hat{h}_t,a_t) = Q^{\pi_{\theta}}(\hat{h}_t,a_t) - V^{\pi_{\theta}}(\hat{h}_t)$ is the so-called \textit{advantage function} \cite{gae}. We implement $\pi_{\theta}$ with a deep neural network that takes as input the summarized history $\hat{h}_t$ and produces as output a discrete conditional probability distribution $\pi_{\theta}(a_t|\hat{h}_t)$ that is computed with the softmax function. The neural network structure of $\pi_{\theta}$ follows an actor-critic architecture and computes a second output (the critic) that estimates the value function $V_{\theta}^{\pi_{\theta}}(\hat{h}_t)$, which in turn allows to estimate $A^{\pi_{\theta}}(\hat{h}_t,a_t)$ in Eq. \ref{eq:pg_objective} using the \textit{generalized advantage estimator} $\hat{A}_{GAE}^{\pi_{\theta}}$ \cite{gae}.

The hyperparameters of our implementation are given in Appendix \ref{appendix:hyperparameters} and were decided based on smaller search in parameter space.

The defender policy is learned through simulation of the POMDP. First, we simulate a given number of episodes. We then use the episode outcomes and trajectories to estimate the expectation of the gradient in Eq. \ref{eq:pg_objective}. Then, we use the estimated gradient and the PPO algorithm \cite{ppo} with the Adam optimizer \cite{adam_opt} to update the policy. This process of simulating episodes and updating the policy continues until the policy has sufficiently converged.
\subsection{Threshold Property of the Optimal Policy}\label{sec:dp_opt}
The policy that solves the optimal stopping problem is defined by the optimization objective in Eq. \ref{eq:rl_prob2}. From the theory of dynamic programming, we know that this policy satisfies the Bellman equation \cite{bellman1957markovian,bert05,krishnamurthy_2016,pomdp_belief_optimal}:
\begin{align}
\pi^{*}\big(b(1)\big) &= \argmax_{a \in \mathcal{A}} \Bigg[r\big(b(1),a\big) + \sum_{o\in \mathcal{O}}\mathbb{P}[o| b(1), a]V^{*}\big(b_o^{a}(1)\big)\Bigg] \label{eq:bellman_belief}
\end{align}
where $b(1) = \mathbb{P}[s_t=1|h_t]$ is the belief that the system is in state $1$ based on the observed history $h_t=(a_1,o_1,\hdots,a_{t-1},o_t)$. Consequently, $b(0)=1-b(1)$ (see Section \ref{sec:mdps} for an overview of belief states). Moreover, $b^{a}_{o}(1)$ is the belief state updated with the Bayes filter in Eq. \ref{eq:belief_upd} after taking action $a$ and observing $o$. Further, $r\big(b(1), a\big)$ is the expected reward of taking action $a$ in belief state $b(1)$, and $V^{*}$ is the value function.

We use Eq. \ref{eq:bellman_belief} to derive properties of the optimal policy. Specifically, we establish the following structural result.
\begin{theorem}\label{thm:structural_result}
There exists an optimal policy $\pi^{*}$ which is a threshold policy of the form:
\begin{align}
\pi^{*}\big(b(1)\big)=
\begin{dcases}
  S \text{ (stop)} & \quad \text{if $b(1) \geq \alpha^{*}$} \\
  C \text{ (continue)}& \quad \text{otherwise}
\end{dcases}\label{eq:structural_result}
\end{align}
where $\alpha^{*}$ is a threshold.
\end{theorem}
\begin{proof}
See Appendix \ref{appendix:structural_result_proof}.
\end{proof}
Theorem \ref{thm:structural_result} states that there exists an optimal policy which stops whenever the posterior probability that an intrusion has started based on the history of IDS alerts and login attempts exceeds a threshold level $\alpha^{*}$. This implies that the optimal policy is completely determined by $\alpha^{*}$ given that $b(1)$ is known. Since $b(1)$ is computed from the history of observations, it also implies that the optimal defender policy can expressed as a threshold function based on the observed infrastructure metrics.
\section{Emulating the Target Infrastructure to Instantiate the Simulation}\label{sec:policy_learning_results}
To simulate episodes of the POMDP we must know the distributions of alerts and login attempts. We estimate these distributions using measurements from an emulation system. This procedure is detailed in this section.
\subsection{Emulating the Target Infrastructure}
The emulation system executes on a cluster of machines that runs a virtualization layer provided by Docker \cite{docker} containers and virtual connections. The emulation is configured following the topology in Fig. \ref{fig:system2} and the configuration in Appendix \ref{appendix:infrastructure_configuration}. It emulates the clients, the attacker, and the defender, as well as $31$ physical components of the target infrastructure (e.g application servers and the gateway). Each physical entity is emulated using a Docker container. The containers replicate important functions of the target infrastructure, including web servers, databases, SSH servers, etc.

The emulation evolves in discrete time-steps of $30$ seconds. During each time-step, the attacker and the defender can perform one action each.
\subsubsection{Emulating the Client Population}
The client population is emulated by three client processes that interact with the application servers through different functions at short intervals, see Table \ref{tab:client_profiles}.
\begin{table}
\centering
\resizebox{0.85\columnwidth}{!}{%
\begin{tabular}{lll} \toprule
  {\textit{Client}} & {\textit{Functions}} & {\textit{Application servers}} \\ \midrule
  $1$ & HTTP, SSH, SNMP, ICMP & $N_2,N_3,N_{10},N_{12}$\\
  $2$ & IRC, PostgreSQL, SNMP & $N_{31},N_{13},N_{14},N_{15},N_{16}$\\
  $3$ & FTP, DNS, Telnet & $N_{10}, N_{22}, N_{4}$ \\
  \bottomrule\\
\end{tabular}
}
\caption{Emulated client population; each client interacts with application servers using a set of functions.}\label{tab:client_profiles}
\end{table}
\subsubsection{Emulating the Attacker}\label{sec:emu_attack}
The start time of an intrusion is controlled by a Bernoulli process as explained in Section \ref{sec:formal_model_2}. We have implemented two types of attackers, \textsc{NoisyAttacker} and \textsc{StealthyAttacker}, both of which execute the sequence of actions listed in Table \ref{tab:static_attackers}. The actions consist of reconnaissance commands and exploits. During each time-step, one action is executed.

The two types of attackers differ in the reconnaissance command. \textsc{NoisyAttacker} uses a TCP/UDP scan for reconnaissance while \textsc{StealthyAttacker} uses a ping-scan. Since the ping-scan generates fewer IDS alerts than the TCP/UDP scan, it makes the actions of \textsc{StealthyAttacker} harder to detect.
\begin{table}
\centering
\resizebox{0.95\columnwidth}{!}{%
\begin{tabular}{ll} \toprule
  {\textit{Time-steps $t$}} & {\textit{Actions}} \\ \midrule
  $1$--$I_t\sim Ge(0.2)$ & (Intrusion has not started) \\
  $I_t+1$--$I_t+7$ & \textsc{Recon}, brute-force attacks (SSH,Telnet,FTP) \\
                           & on $N_2,N_{4},N_{10}$, login($N_2,N_4,N_{10}$), \\
                           & backdoor($N_2,N_4,N_{10}$), \textsc{Recon}\\
  $I_t+8$--$I_t+11$ & CVE-2014-6271 on $N_{17}$, SSH brute-force attack on $N_{12}$, \\
  & login ($N_{17},N_{12}$), backdoor($N_{17},N_{12}$)\\
  $I_t+12$--$X+16$ & CVE-2010-0426 exploit on $N_{12}$, \textsc{Recon}\\
  &  SQL-Injection on $N_{18}$, login($N_{18}$), backdoor($N_{18}$)\\
  $I_t+17$--$I_t+22$ & \textsc{Recon}, CVE-2015-1427 on $N_{25}$, login($N_{25}$)\\
  & \textsc{Recon}, CVE-2017-7494 exploit on $N_{27}$, login($N_{27}$)\\
  \bottomrule\\
\end{tabular}
}
\caption{Attacker actions to emulate an intrusion.}\label{tab:static_attackers}
\end{table}
\subsubsection{Emulating Actions of the Defender}
The defender takes an action every time-step. The continue action has no effect on the emulation. The stop action changes the firewall configuration of the gateway and drops all incoming traffic.
\subsection{Estimating the Distributions of Alerts and Login Attempts}
In this section, we describe how we collect data from the emulation and how we use the data to estimate the distributions of alerts and login attempts.
\begin{table}
  \centering
\begin{tabular}{ll} \toprule
  {\textit{Metric}} & {\textit{Command in the Emulation}} \\ \midrule
  Login attempts & \texttt{cat /var/log/auth.log}\\
  IDS Alerts & \texttt{cat /var/snort/alert.csv}\\
  \bottomrule\\
\end{tabular}
\caption{Commands used to measure the emulation.}\label{tab:attribute_sensors}
\end{table}
\begin{figure}
  \centering
    \scalebox{0.85}{
      \includegraphics{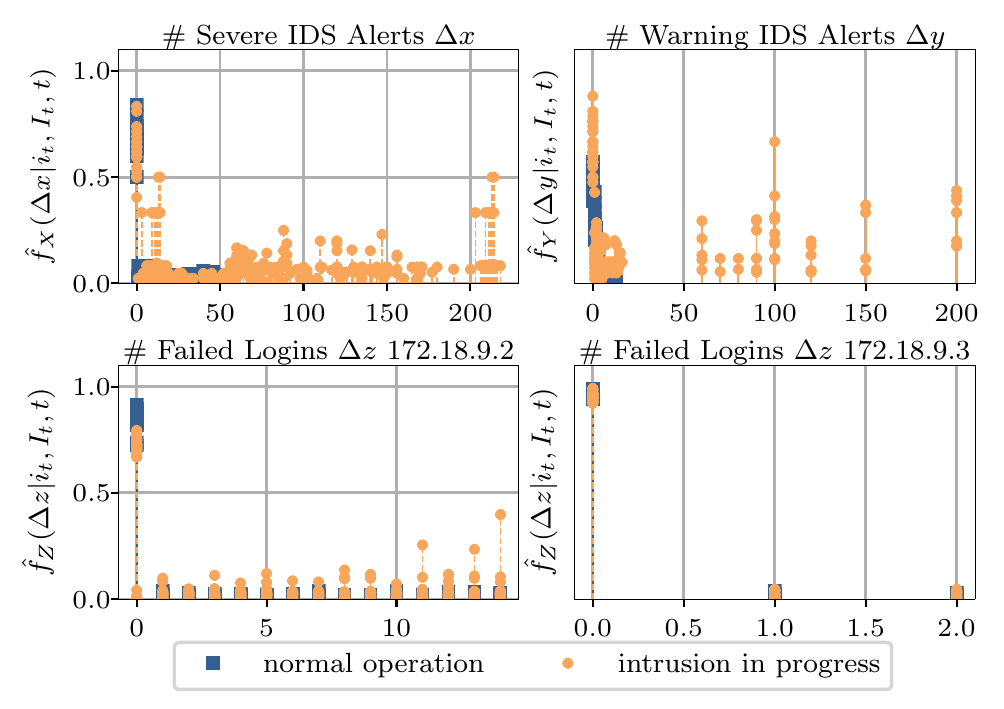}
    }
    \caption{Empirical distributions of IDS alerts (top row) and login attempts on two servers (bottom row); the graphs include several distributions that are superimposed.}
    \label{fig:ids_infra_one_macine_22}
  \end{figure}
\begin{figure*}
  \centering
\scalebox{0.98}{
      \includegraphics{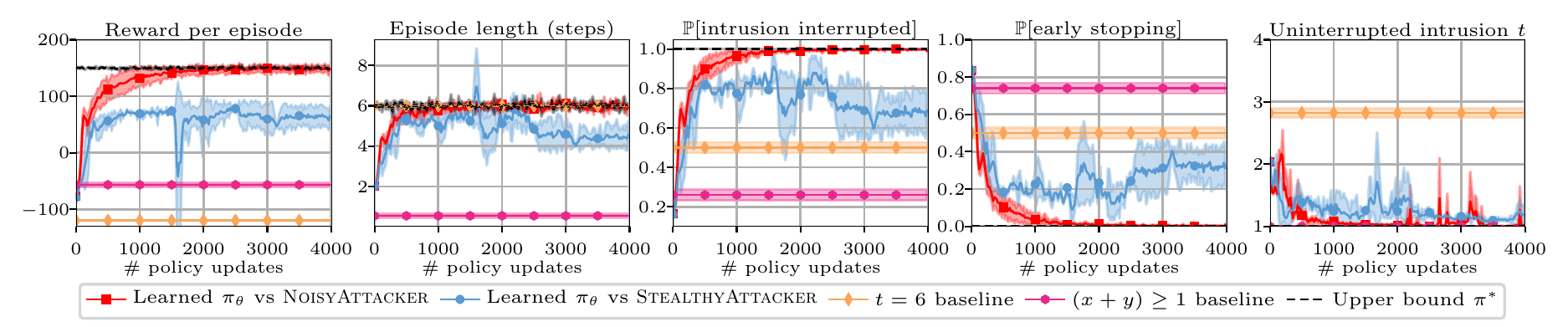}
    }
    \caption{Learning curves; the graphs show from left to right: episodic reward, length of an episode, empirical detection probability, empirical early stopping probability, and the number of steps between the start of an intrusion and the stop action; the curves show the averages and the standard deviations of three training runs with different random seeds.}
    \label{fig:training_results_1}
  \end{figure*}
\subsubsection{Measuring the Number of IDS alerts and Login Attempts in the Emulation}
At the end of every time-step, the emulation system collects the metrics $\Delta x$, $\Delta y$, $\Delta z$, which contain the alerts and login attempts that occurred during the time-step. The metrics are collected by parsing the output of the commands in Table \ref{tab:attribute_sensors}. For the evaluation reported in this paper, we collected measurements from $11000$ time-steps.
\subsubsection{Estimating the Distributions of Alerts and Login Attempts of the Target Infrastructure}
Using the collected measurements, we compute the empirical distribution $\hat{f}_{XYZ}$, which is our estimate of the corresponding distribution $f_{XYZ}$ in the target infrastructure. For each $(I_t, t)$ pair, we obtain one empirical distribution.

Fig. \ref{fig:ids_infra_one_macine_22} shows some of these distributions, which are superimposed. Although the distribution patterns generated during an intrusion and during normal operation overlap, there is a clear difference.
\subsection{Simulating Episodes of the POMDP}
During a simulation of the POMDP, the system state evolves according the dynamics described in Section \ref{sec:formal_model_2} and the observations evolve according to the estimated distribution $\hat{f}_{XYZ}$. In the initial state, no intrusion occurs. In every episode, either the defender stops before the intrusion starts or exactly one intrusion occurs, the start of which is determined by a Bernoulli process (see Section \ref{sec:formal_model_2}).

A simulated episode evolves as follows. During each time-step, if an intrusion is ongoing, the attacker executes an action in the predefined sequence listed in Table \ref{tab:static_attackers}. Subsequently, the defender samples an action from the defender policy $\pi_{\theta}$. If the action is stop, the episode ends. Otherwise, the simulation samples the number of alerts and login attempts occurring during this time-step from the empirical distribution $\hat{f}_{XYZ}$. It then computes the reward of the defender using the reward function defined in Section \ref{sec:formal_model_2}. The activities of the clients are not explicitly simulated but are implicitly represented in $\hat{f}_{XYZ}$. The sequence of time-steps continues until the defender stops or an intrusion completes, after which the episode ends.
\section{Learning Intrusion Prevention Policies using Simulation}\label{sec:eval}
To evaluate our reinforcement learning approach for finding defender policies, we simulate episodes of the POMDP where the defender policy is updated and evaluated. We evaluate the approach with respect to the convergence of policies and compare the learned policies to two baselines and to an ideal policy which presumes knowledge of the exact time of intrusion.

The evaluation is conducted using a Tesla P100 GPU and the hyperparameters for the learning algorithm are listed in Appendix \ref{appendix:hyperparameters}. Our implementation as well as the measurements for the results reported in this paper are publicly available \cite{github_cnsm_21_hammar_stadler}.
\subsection{Evaluation Setup}
We train two defender policies against \textsc{NoisyAttacker} and \textsc{StealthyAttacker} until convergence, which occurs after some $400$ iterations. In each iteration, we simulate $4000$ time-steps and perform $10$ updates to the policy. After each iteration, we evaluate the defender policy by simulating $200$ evaluation episodes and compute various performance metrics.

We compare the learned policies with two baselines. The first baseline is a policy that always stops at $t=6$, which corresponds to the immediate time-step after the expected time of intrusion $\mathbb{E}[I_t]=5$ (see Section \ref{sec:formal_model_2}). The policy of the second baseline always stops after the first IDS alert occurs, i.e. $(x+y)\geq 1$.

To evaluate the stability of the learning curves' convergence, we run each training process three times with different random seeds. One training run requires approximately six hours of processing time on a P100 GPU.

\subsection{Analyzing the Results}
The red curves in Fig. \ref{fig:training_results_1} show the performance of the learned policy against \textsc{NoisyAttacker}, and the blue curves show the policy performance against \textsc{StealthyAttacker}. The purple and the orange curves give the performance of the two baseline policies. The dashed black curves give an upper bound on the performance of the optimal policy $\pi^{*}$, which is computed assuming that the defender knows the exact time of intrusion.

The five graphs in Fig. \ref{fig:training_results_1} show that the learned policies converge, and that they are close to optimal in terms of achieving maximum reward, detecting intrusion, avoid stopping when there is no intrusion, and stopping right after an intrusion starts. Further, the learned policies outperform both baselines by a large margin (leftmost graph in Fig. \ref{fig:training_results_1}).

The performance of the learned policy against \textsc{NoisyAttacker} is better than that against \textsc{StealthyAttacker} (leftmost graph in Fig. \ref{fig:training_results_1}). This indicates that \textsc{NoisyAttacker} is easier to detect for the defender. For instance, the learned policy against \textsc{StealthyAttacker} has a higher probability of stopping early (second rightmost graph of Fig. \ref{fig:training_results_1}). This can also be seen in the second leftmost graph of Fig. \ref{fig:training_results_1}, which shows that, on average, the learned policy against \textsc{StealthyAttacker} stops after $5$ time-steps and the learned policy against \textsc{NoisyAttacker} stops after $6$ time-steps.

Looking at the baseline policies, we can see that the baseline $t=6$ stops too early in $50$ percent of the episodes, and the $(x+y)\geq 1$ baseline stops too early in $80$ percent of the episodes (second rightmost graph in Fig. \ref{fig:training_results_1}).

\section{Threshold Properties of the Learned Policies and Comparison with the Optimal Policy}\label{sec:prop_policies}
\begin{figure}
  \centering
  \scalebox{0.46}{
    \input{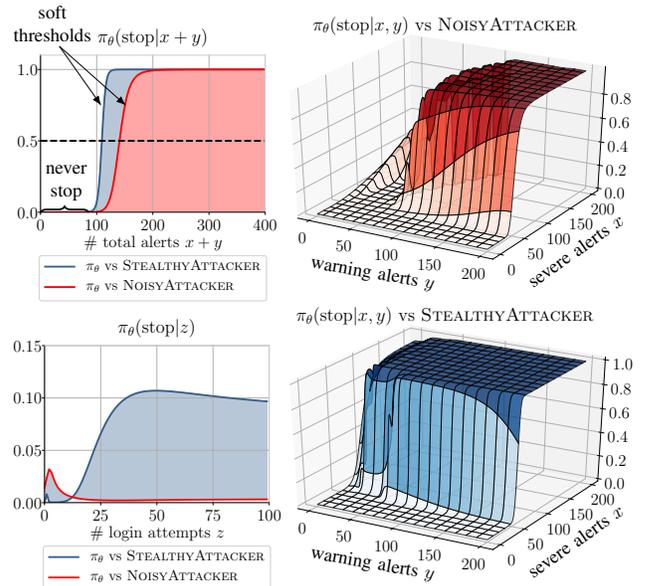}
    }
    \caption{Probability of the stop action by the learned policies in function of the number of alerts $x,y$ and login attempts $z$.}
    \label{fig:policy_inspection_def}
  \end{figure}

When analyzing the learned policies, we find that they can be expressed through thresholds, just like the optimal policy (Section \ref{sec:dp_opt}). However, in contrast to the optimal policy, the learned thresholds are based on the observed counters of alerts and login attempts rather than the belief state (which is unknown to the defender). Specifically, the top left graph in Fig. \ref{fig:policy_inspection_def} shows that the learned policies against both attackers implement a soft threshold by stopping with high probability if the number of alerts ($x_t + y_t$) exceeds $130$. This indicates that if the total number of alerts is above $130$, an intrusion has started, i.e. $x_t + y_t$ is used to approximate the posterior $b_t(1)$.

Moreover, the graphs in Fig. \ref{fig:policy_inspection_def} also show the relative importance of severe alerts $x_t$, warning alerts $y_t$, and login attempts $z_t$ for policy decisions. Specifically, it can be seen that $x_t$ has the highest importance, $y_t$ has a lower importance, and $z_t$ has the least importance for policy decisions.

We also see that the learned policy against \textsc{NoisyAttacker} is associated with a higher alert threshold than that of \textsc{StealthyAttacker} (top left graph in Fig. \ref{fig:policy_inspection_def}). This is consistent with our comment in Section \ref{sec:emu_attack} that \textsc{StealthyAttacker} is harder to detect.

Lastly, Fig. \ref{fig:opt_learned_policies.pdf} suggest that the thresholds of the learned policies are indeed close to the threshold of the optimal policy. For instance, the policy learned against \textsc{NoisyAttacker} stops immediately after the optimal stopping time.
\begin{figure}
  \centering
  \scalebox{0.59}{
    \input{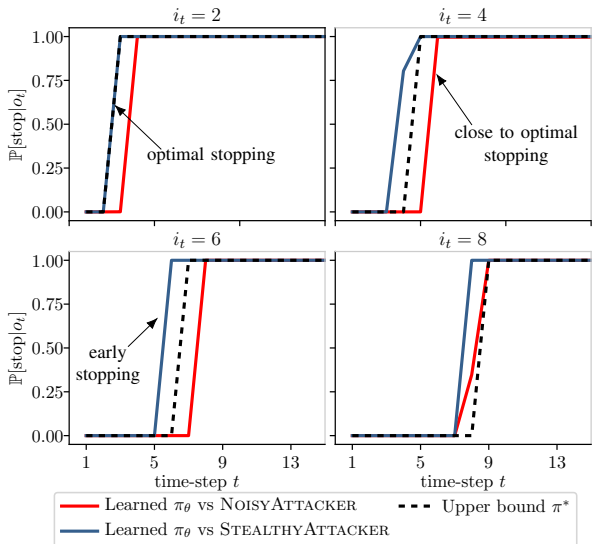}
    }
    \caption{Comparison of the learned policies $\pi_{\theta}$ and an upper bound on the optimal policy $\pi^{*}$ for $4$ sample episodes where the intrusions start at $i_t=2,4,6,8$.}
    \label{fig:opt_learned_policies.pdf}
  \end{figure}
\section{Related Work}\label{sec:related_work}
The problem of automatically finding security policies has been studied using concepts and methods from different fields, most notably reinforcement learning, game theory, dynamic programming, control theory, attack graphs, statistical tests, and evolutionary computation. For a literature review of deep reinforcement learning in network security see \cite{deep_rl_cyber_sec}, and for an overview of game-theoretic approaches see \cite{nework_security_alpcan}. For examples of research using dynamic programming, control theory, attack graphs, statistical tests, and evolutionary methods, see \cite{dp_security_1}, \cite{malware_oc}, \cite{miehling_attack_graph}, \cite{tartakovsky_1}, and \cite{armsrace_malware}.
  %%The problem of automating security policies have been studied in different fields, most notably the fields of reinforcement learning, game theory, dynamic programming, control theory, attack graphs, statistical tests, and evolutionary computation. For a literature review of deep reinforcement learning in network security, see \cite{deep_rl_cyber_sec} and for an overview of game-theoretic approaches see \cite{nework_security_alpcan}. For examples of approaches using dynamic programming, control theory, attack graphs, statistical tests, and evolutionary methods, see \cite{dp_security_1}, \cite{malware_oc}, \cite{mal_pontus,miehling_attack_graph}, \cite{tartakovsky_1}, and \cite{armsrace_malware}.

Most research on reinforcement learning applied to network security is recent. Prior work that most resembles the approach taken in this paper includes our previous research \cite{hammar_stadler} and the work in \cite{elderman}, \cite{schwartz_2020}, \cite{oslo_uni}, \cite{kurt_rl}, \cite{microsoft_red_teaming}, and \cite{ridley_ml_defense}. All of these papers focus on network intrusions using reinforcement learning.

This paper differs from prior work in the following ways: (1) we formulate intrusion prevention as an optimal stopping problem (\cite{kurt_rl} uses a similar approach); (2) we use an emulated infrastructure to estimate the parameters of our simulation model, rather than relying on abstract assumptions like \cite{hammar_stadler,elderman,oslo_uni,kurt_rl,microsoft_red_teaming,ridley_ml_defense}; (3) we derive a structural property of the optimal policy; (4) we analyze the learned policies and relate them to the optimal policy, an analysis which prior work lacks \cite{hammar_stadler,elderman,oslo_uni,microsoft_red_teaming,ridley_ml_defense}; and (5) we apply state-of-the-art reinforcement learning algorithms, i.e. PPO \cite{ppo}, rather than traditional ones as used in \cite{elderman,oslo_uni,kurt_rl,schwartz_2020,microsoft_red_teaming,ridley_ml_defense}.

Optimal stopping has previously been applied to model problems in in different domains, including finance \cite{optimal_stopping_finance}, queuing systems \cite{roy_threshold}, and attack detection \cite{tartakovsky_1}. We believe we are first in using optimal stopping to model intrusion prevention.
%%Optimal stopping has previously been used to model problems in domains such as finance \cite{optimal_stopping_finance}, queuing systems \cite{roy_threshold}, and attack detection \cite{tartakovsky_1,kurt_rl}. Although we take inspiration from these works, our work differs in that we use optimal stopping to model intrusion prevention, which we are, to the best of our knowledge, the first to do.
%%Lastly, to address the computational challenge solving optimal stopping problems over large state spaces, the most common approach used in prior work, similar to the approach taken in this paper, is to use reinforcement learning with function approximation \cite{roy_threshold, optimal_stopping_rl_1}. The work in this paper differ from the prior work in this regard by using a different reinforcement algorithm.
%%we are not the first to recognize the computational challenge of solving optimal stopping problems over large state spaces. To address this challenge,
\section{Conclusion and Future Work}\label{sec:conclusions}
In this paper, we proposed a novel formulation of the intrusion prevention problem as one of optimal stopping. This allowed us to state that the optimal defender policy can be expressed using a threshold obtained from infrastructure measurements. Further, we used reinforcement learning to estimate the optimal defender policies in a simulation environment. In addition to validating the predictions from the theory, we learned from the simulations a) the relative importance of measurement metrics with respect to the threshold level and b) that different attacker profiles can lead to different thresholds of the defender policies.

We plan to extend this work in three directions. First, the model of the defender in this paper is simplistic as it allows only for a single stop action. We plan to increase the set of actions that the defender can take to better reflect today's defense capabilities, while still keeping the structure of the stopping formulation. Second, we plan to extend the observation capabilities of the defender to obtain more realistic policies. Third, in the current paper, the attacker policy is static. We plan to extend the model to include a dynamic attacker that can learn just like the defender. This requires a game-theoretic formulation of the problem.

\section{Acknowledgments}
This research has been supported in part by the Swedish armed forces and was conducted at KTH Center for Cyber Defense and Information Security (CDIS). The authors would like to thank Pontus Johnson for useful input to this research, and Forough Shahab Samani and Xiaoxuan Wang for their constructive comments to an earlier draft of this paper.

\appendix
\subsection{Proof of Theorem \ref{thm:structural_result}}\label{appendix:structural_result_proof}
We will work our way to the proof of Theorem \ref{thm:structural_result} by establishing some initial results.
\begin{lemma}\label{lemma:belief_condition}
It is optimal to stop in belief state $b(1)$ iff:
\footnotesize\begin{align}
               &b(1) \geq \\
               &\frac{110 + \displaystyle\sum_{o\in \mathcal{O}} V^{*}\big(b_o^{C}(1)\big)\Big(p\mathcal{Z}(o,1,C) + (1-p)\mathcal{Z}(o,0,C)\Big)}{300 + \displaystyle\sum_{o\in \mathcal{O}}V^{*}\big(b_o^{C}(1)\big)\Big(p\mathcal{Z}(o,1,C) + (1-p)\mathcal{Z}(o,0,C)- \mathcal{Z}(o,1,C)\Big)}\nonumber
\end{align}\normalsize
\end{lemma}
\begin{proof}
Considering both actions of the defender ($\mathcal{A} = \{S,C\}$), we derive from the Bellman equation (Eq. \ref{eq:bellman_belief}):\footnotesize
\begin{align}
&\pi^{*}\big(b(1)\big) \\
  &= \argmax_{a \in \mathcal{A}} \Bigg[r\big(b(1),a\big) + \sum_{o\in \mathcal{O}}\mathbb{P}[o| b(1), a]V^{*}\big(b_o^{a}(1)\big)\Bigg]\nonumber \\
  &= \argmax \Bigg[\underbrace{r\big(b(1),S\big)}_{a=S}, \underbrace{r\big(b(1),C\big) + \sum_{o\in \mathcal{O}}\mathbb{P}[o| b(1), C]V^{*}\big(b_o^{C}(1)\big)}_{a=C}\Bigg]\nonumber\\
&=\argmax \Bigg[\underbrace{b(1)200 - 100}_{\text{\normalsize $\omega$}}, \underbrace{10-b(1)100 + \sum_{o\in \mathcal{O}}\mathbb{P}[o| b(1), C]V^{*}\big(b_o^{C}(1)\big)}_{\text{\normalsize $\epsilon$}}\Bigg]\nonumber
\end{align}\normalsize
In the above equation, $\omega$ is the expected reward for stopping and $\epsilon$ is the expected cumulative reward for continuing. If $\epsilon = \omega$, both actions of the defender, continuing and stopping, maximize the expected cumulative reward. If $\omega \geq \epsilon$, it is optimal for the defender to stop.

Next, we use $\mathbb{P}[o| b(1), a] = \sum_{s \in \mathcal{S}} \sum_{s_{\prime}\in \mathcal{S}} b(s)\mathcal{P}^a_{ss^{\prime}}\mathcal{Z}(o,s^{\prime}, a)$ and $\mathcal{S} = \{0,1\}$ to obtain:
\footnotesize
\begin{align}
&\pi^{*}\big(b(1)\big) \\
  &=\argmax \Bigg[b(1)200 - 100, 10-b(1)100 + \sum_{o\in \mathcal{O}}\mathbb{P}[o| b(1), C]V^{*}\big(b_o^{C}(1)\big)\Bigg]\nonumber\\
&=\argmax \Bigg[b(1)200 - 100, 10-b(1)100 + \sum_{o\in \mathcal{O}}V^{*}\big(b_o^{C}(1)\big) \nonumber\\
  &\bigg(b(1)\mathcal{Z}(o,1,C) + (1-b(1)\big)\Big(p\mathcal{Z}(o,1,C) + (1-p)\mathcal{Z}(o,0,C)\Big)\bigg)\Bigg]\nonumber\\
&=\argmax \Bigg[b(1)200 - 100, 10 + b(1)\bigg(-100 + \sum_{o\in \mathcal{O}}V^{*}\big(b_o^{C}(1)\big)\Big(\mathcal{Z}(o,1,C)\nonumber \\
  &-\big(p\mathcal{Z}(o,1,C) + (1-p)\mathcal{Z}(o,0,C)\big)\Big)\bigg) \nonumber\\
&+ \sum_{o\in \mathcal{O}} V^{*}\big(b_o^{C}(1)\big)\bigg(p\mathcal{Z}(o,1,C) + (1-p)\mathcal{Z}(o,0,C)\bigg) \Bigg] \nonumber
\end{align}\normalsize
This implies that it is optimal to stop in belief state $b(1)$ iff:
\footnotesize\begin{align}
&b(1)200 - 100 \geq 10 + b(1)\bigg(-100 + \sum_{o\in \mathcal{O}}V^{*}\big(b_o^{C}(1)\big)\Big(\mathcal{Z}(o,1,C) \nonumber\\
  &-\big(p\mathcal{Z}(o,1,C) + (1-p)\mathcal{Z}(o,0,C)\big)\Big)\bigg) \nonumber\\
&+ \sum_{o\in \mathcal{O}} V^{*}\big(b_o^{C}(1)\big)\bigg(p\mathcal{Z}(o,1,C) + (1-p)\mathcal{Z}(o,0,C)\bigg)
\end{align}\normalsize
By rearranging terms, we get:
\footnotesize\begin{align}
               &b(1) \geq \\
               &\underbrace{\frac{110 + \displaystyle\sum_{o\in \mathcal{O}} V^{*}\big(b_o^{C}(1)\big)\Big(p\mathcal{Z}(o,1,C) + (1-p)\mathcal{Z}(o,0,C)\Big)}{300 + \displaystyle\sum_{o\in \mathcal{O}}V^{*}\big(b_o^{C}(1)\big)\Big(p\mathcal{Z}(o,1,C) + (1-p)\mathcal{Z}(o,0,C)- \mathcal{Z}(o,1,C)\Big)}}_{\text{\normalsize $\alpha_{b(1)}$}}\nonumber
\end{align}\normalsize
\end{proof}
Lemma \ref{lemma:belief_condition} shows that the optimal policy is determined by the scalar thresholds $\alpha_{b(1)}$. Specifically, it is optimal to stop in belief state $b(1)$ if $b(1) \geq \alpha_{b(1)}$. We conclude that the \textit{stopping set} $\mathscr{S}$\textemdash the set of belief states $b(1) \in [0,1]$ where it is optimal to stop\textemdash is:
\begin{align}
\mathscr{S} = \left\{b(1) \in [0,1] : b(1) \geq \alpha_{b(1)} \right\}
\end{align}
Similarly, the \textit{continuation set} $\mathscr{C}$\textemdash the set of belief states where it is optimal to continue\textemdash is $\mathscr{C} = [0,1] \setminus \mathscr{S}$.

Building on the above analysis, the main idea behind the proof of Theorem \ref{thm:structural_result} is to show that the stopping set $\mathscr{S}$ has the form $\mathscr{S} = [\alpha^{*}, 1]$, where $0 \leq \alpha^{*} \leq 1$ is the stopping threshold. Towards this goal, we state the following two lemmas.
\begin{lemma}\label{lemma:convex_value_fun}
The following lemma is due to Sondik \cite{smallwood_1}.

The optimal value function:
\begin{align}
V^{*}(b_t) = \max_{a_t\in \mathcal{A}} \mathbb{E}\left[r_{t+1} + V^{*}(b_{t+1}) | b_t, a_t\right]
\end{align}
is piecewise linear and convex with respect to $b \in \mathcal{B}$.
\end{lemma}
\begin{proof}
See \cite{smallwood_1} or \cite[pp. 155, Theorem 7.4.1]{krishnamurthy_2016}.
\end{proof}
\begin{lemma}\label{lemma:convex_stopping_set}
The stopping set $\mathscr{S}$ is a convex subset of the belief space $\mathcal{B}$.
\end{lemma}
\begin{proof}
A general proof is given in \cite{krishnamurthy_2016}. We restate it here to show that it holds in our case.

To show that the stopping set $\mathscr{S}$ is convex, we need to show that for any two belief states $b_1,b_2 \in \mathscr{S}$, any linear combination of $b_1,b_2$ is also in $\mathscr{S}$. That is, $b_1,b_2\in \mathscr{S} \implies \lambda b_1 + (1-\lambda)b_2 \in \mathscr{S}$ for any $\lambda \in [0,1]$.

Since $V^{*}(b)$ is convex (Lemma \ref{lemma:convex_value_fun}), we have by definition of convex sets that:
\begin{align}
V^{*}(\lambda b_1 + (1-\lambda)b_2) &\leq \lambda V^{*}(b_1) + (1-\lambda)V(b_2)
\end{align}
Further, as $b_1,b_2 \in \mathscr{S}$ by assumption, the optimal action in $b_1$ and $b_2$ is the stop action $S$. Thus, we have that $V^{*}(b_1)=Q^{*}(b_1,S)=100b_1(1)-100b_1(0)$ and $V^{*}(b_2)=Q^{*}(b_2,S)=100b_2(1) - 100b_2(0)$. Hence:
\footnotesize\begin{align}
  &V^{*}\big(\lambda b_1(1) + (1-\lambda)b_2(1)\big) \\
  &\leq \lambda V^{*}\big(b_1(1)) + (1-\lambda)V^{*}(b_2(1)\big)\\
&= \lambda Q^{*}(b_1, S) + (1-\lambda)Q^{*}(b_2, S)\\
&= \lambda\big(100b_1(1) - b_1(0)100\big) + (1-\lambda)\big(100b_2(1) - b_2(0)100\big)\\
&= Q^{*}\big(\lambda b_1 + (1-\lambda)b_2, S\big)\\
&\leq V^{*}\big(\lambda b_1(1) + (1-\lambda)b_2(1)\big)
\end{align}\normalsize
where the last inequality is because $V^{*}$ is optimal. Thus we have that $Q^{*}\big(\lambda b_1 + (1-\lambda)b_2, S\big) = V^{*}\big(\lambda b_1(1) + (1-\lambda)b_2(1)\big)$. This means that if $b_1,b_2 \in \mathscr{S}$, then $\lambda b_1 + (1-\lambda)b_2 \in \mathscr{S}$ for any $\lambda \in [0,1]$. Hence, $\mathscr{S}$ is convex.
\end{proof}
Now we use Lemma \ref{lemma:convex_stopping_set} to prove Theorem \ref{thm:structural_result}.
\begin{proof}[Proof of Theorem \ref{thm:structural_result}]
The proof of Theorem \ref{thm:structural_result} follows the same argument as the proof in \cite[Corollary 12.2.2, pp. 258]{krishnamurthy_2016}.

The belief space $\mathcal{B} = [0,1]$ is defined by $b(1) \in [0,1]$. In consequence, using Lemma \ref{lemma:convex_stopping_set}, we have that the stopping set $\mathscr{S}$ is a convex subset of $[0,1]$. That is, $\mathscr{S}$ has the form $[\alpha^{*}, \beta^{*}]$ where $0 \leq \alpha^{*} \leq \beta^{*} \leq 1$. Thus, to show that the optimal policy is of the form:
\begin{align}
\pi^{*}\big(b(1)\big)=
\begin{dcases}
  S & \quad \text{if $b(1) \geq \alpha^{*}$} \\
  C & \quad \text{otherwise}
\end{dcases}
\end{align}
it suffices to show that $\beta^{*} = 1$, i.e. $\mathscr{S} = [\alpha^{*}, \beta^{*}] = [\alpha^{*}, 1]$.

If $b(1) = 1$, then the Bellman equation states that:
\footnotesize\begin{align}
\pi^{*}(1)&= \argmax_{a \in \mathcal{A}} \Bigg[r(1,a) + \sum_{o\in \mathcal{O}}\mathbb{P}[o| 1, a]V^{*}\big(b_o^{a}(1)\big)\Bigg]\\
&= \argmax \Bigg[\underbrace{100}_{a=S}, \underbrace{-90 + \sum_{o\in \mathcal{O}}\mathcal{Z}(o,1,C)V^{*}\big(b_o^{C}(1)\big)}_{a=C}\Bigg]
\end{align}\normalsize
Since $s=1$ is an absorbing state until the terminal state is reached, we have that $b_o^{C}(1)=1$ for all $o \in \mathcal{O}$. This follows from the definition of $b_o^{C}$ (Eq. \ref{eq:belief_upd}). Consequently, we get:
\small\begin{align}
  \pi^{*}(1) &= \argmax \Bigg[100, -90 + \sum_{o\in \mathcal{O}}\mathcal{Z}(o,1,C)V^{*}\big(b_o^{C}(1)\big)\Bigg]\\
             &= \argmax \Bigg[100, -90 + V^{*}(1)\Bigg]
\end{align}\normalsize
Finally, since $V^{*}(1) \leq 100$, we conclude that:
\begin{align}
  \pi^{*}(1) &= \argmax \Bigg[100, -90 + V^{*}(1)\Bigg] =S
\end{align}
This means that $\pi^{*}(1) = S$, hence $b(1)=1$ is in the stopping set, i.e. $1 \in \mathscr{S}$. As $1 \in \mathscr{S}$, and since $\mathscr{S}$ is a convex subset $[\alpha^{*}, \beta^{*}] \subseteq [0,1]$, we have that $\beta^{*}=1$. Then it follows that $\mathscr{S} = [\alpha^{*}, \beta^{*}] = [\alpha^{*}, 1]$.
\end{proof}

\paragraph*{An Example to Illustrate Theorem \ref{thm:structural_result}}
To illustrate the implications of Theorem \ref{thm:structural_result}, consider the following example.

The observation $o_t$ is the number of IDS alerts that were generated during time-step $t$, which is an integer scalar in the observation space $\mathcal{O} = \{0,\hdots,5\}$. Further, assume that the observation function $\mathcal{Z}(o^{\prime},s^{\prime},a)$ is defined using the discrete uniform distribution $\mathcal{U}(\{a,b\})$ as follows.
\begin{align}
\mathcal{Z}(o^{\prime},0,C) &= \mathcal{U}(\{0,4\}) && \text{no intrusion}\\
\mathcal{Z}(o^{\prime},1,C) &= \mathcal{U}(\{0,5\}) && \text{intrusion}\\
\mathcal{Z}(\emptyset,\emptyset,\cdot) &= 1
\end{align}
The rest of the POMDP follows the definitions in Section \ref{sec:formal_model_2}.

Due to the small observation space, the optimal policy can be computed using dynamic programming and value iteration. In particular, we apply Sondik's value iteration algorithm \cite{smallwood_1} to compute the optimal value function $V^{*}\big(b(1)\big)$ as well as the optimal thresholds $\alpha_{b(1)}$ (Fig. \ref{fig:proof_example}).
\begin{figure}
  \centering
  \scalebox{0.63}{
    \input{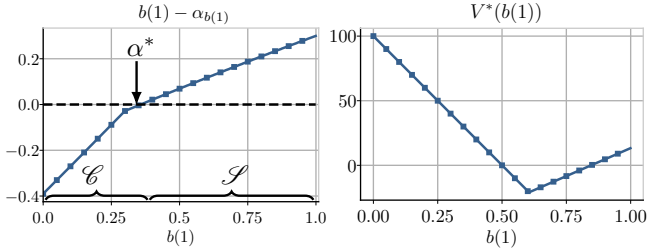}
    }
    \caption{Left: optimal stopping thresholds $b(1)-\alpha_{b(1)}$, if $b(1)-\alpha_{b(1)} \geq 0$ it is optimal to stop; right: the piecewise linear and convex optimal value function $V^{*}(b(1))$.}
    \label{fig:proof_example}
  \end{figure}

As can be seen in Fig. \ref{fig:proof_example}, $b(1)-\alpha_{b(1)}$ is increasing in $b(1)$ and there exists a unique minimum belief point $b(1) \approx 0.357$ such that $b(1) \geq \alpha_{b(1)}$, which we denote by $\alpha^{*}$. Hence the stopping set is the convex set $\mathscr{S}=[0.357,1]$, and the continuation set $\mathscr{C}$ is the set $\mathscr{C}=[0,0.357)$.
\subsection{Hyperparameters: Table \ref{tab:hyperparams}}\label{appendix:hyperparameters}
\begin{table}
\centering
\resizebox{1\columnwidth}{!}{%
\begin{tabular}{ll} \toprule
  {\textit{Parameters}} & {\textit{Values}} \\ \midrule
  $\gamma$, lr $\alpha$, batch, \# layers, \# neurons, clip $\epsilon$ & $1$, $5\cdot 10^{-4}$, $4\cdot 10^3$, $3$, $64$, $0.2$\\
  $X_{max}, Y_{max},Z_{max}$, GAE $\lambda$, ent-coef, activation & $10^3$, $10^3$, $10^3$, $0.95$, $5\cdot 10^{-4}$, ReLU \\
  \bottomrule\\
\end{tabular}
}
\caption{Hyperparameters of the learning algorithm.}\label{tab:hyperparams}
\end{table}
\subsection{Configuration of the Infrastructure in Fig. \ref{fig:system2}: Table \ref{tab:emulation_setup}}\label{appendix:infrastructure_configuration}
\begin{table}
\centering
\resizebox{0.96\columnwidth}{!}{%
\begin{tabular}{ll} \toprule
  {\textit{ID (s)}} & {\textit{OS:Services:Exploitable Vulnerabilities}} \\ \midrule
  $1$ & Ubuntu20:Snort(community ruleset v2.9.17.1),SSH:- \\
  $2$ & Ubuntu20:SSH,HTTP Erl-Pengine,DNS:SSH-pw\\
  $4$ & Ubuntu20:HTTP Flask,Telnet,SSH:Telnet-pw \\
  $10$ &Ubuntu20:FTP,MongoDB,SMTP,Tomcat,Teamspeak3,SSH:FTP-pw \\
  $12$ & Jessie:Teamspeak3,Tomcat,SSH:CVE-2010-0426,SSH-pw \\
  $17$ & Wheezy:Apache2,SNMP,SSH:CVE-2014-6271 \\
  $18$ & Deb9.2:IRC,Apache2,SSH:SQL Injection \\
  $22$ & Jessie:PROFTPD,SSH,Apache2,SNMP:CVE-2015-3306 \\
  $23$ & Jessie:Apache2,SMTP,SSH:CVE-2016-10033 \\
  $24$ & Jessie:SSH:CVE-2015-5602,SSH-pw \\
  $25$ & Jessie: Elasticsearch,Apache2,SSH,SNMP:CVE-2015-1427\\
  $27$ & Jessie:Samba,NTP,SSH:CVE-2017-7494\\
  $3$,$11$,$5$-$9$& Ubuntu20:SSH,SNMP,PostgreSQL,NTP:-\\
  $13$-$16$,$19$-$21$,$26$,$28$-$31$& Ubuntu20:NTP, IRC, SNMP, SSH, PostgreSQL:-\\
  \bottomrule\\
\end{tabular}
}
\caption{Configuration of the target infrastructure (Fig. \ref{fig:system2}).}\label{tab:emulation_setup}
\end{table}
\bibliographystyle{IEEEtran}
\bibliography{references,url}
%%\bibliography{url}

\end{document}